\documentclass[10pt,twocolumn,letterpaper]{article}

\usepackage{cvprfinal}
\usepackage{times}
\usepackage{epsfig}
\usepackage[utf8]{inputenc}
\usepackage{color}
\usepackage{bm}
\usepackage{amsfonts,mathtools,amsmath, amsthm}
\usepackage{graphicx}
\usepackage{mwe}
\usepackage{makecell}
\usepackage{booktabs}
\usepackage{enumitem}
\usepackage{pifont} 
\usepackage{float} 
\usepackage{graphicx} 
\graphicspath{{images/}} 

\usepackage[ruled,vlined]{algorithm2e} 
\SetKwRepeat{Do}{do}{while}

\definecolor{cvprblue}{rgb}{0.21,0.49,0.74}
\usepackage[pagebackref,breaklinks,colorlinks,allcolors=cvprblue]{hyperref}

\usepackage{flushend}
\usepackage{setspace}
\usepackage{comment}
\usepackage[capitalize]{cleveref}
\usepackage{multirow}

\def\paperID{1183} 
\def\confName{CVPR}
\def\confYear{2025}

\title{A Flag Decomposition for Hierarchical Datasets}

\author{Nathan Mankovich\\
Universitat de València
\and
Ignacio Santamaria\\
Universidad de Cantabria
\and
Gustau Camps-Valls\\
Universitat de València
\and
Tolga Birdal\\
Imperial College London
}

\newcommand{\R}{\ensuremath{\mathbb{R}}}  
\newcommand{\Gr}{\ensuremath{\mathrm{Gr}}}  

\newcommand{\tr}{\mathrm{tr}}

\newcommand{\Y}{\mathbf{Y}}
\newcommand{\X}{\mathbf{X}}

\newcommand{\x}{\mathbf{x}}

\newcommand{\U}{\mathbf{U}}
\newcommand{\M}{\mathbf{M}}
\newcommand{\Q}{\mathbf{Q}}
\newcommand{\bR}{\mathbf{R}}

\newcommand{\V}{\mathbf{V}}

\newcommand{\B}{\mathbf{B}}
\newcommand{\C}{\mathbf{C}}

\newcommand{\D}{\mathbf{D}}
\newcommand{\bP}{\mathbf{P}}
\newcommand{\cA}{\mathcal{A}}
\newcommand{\cB}{\mathcal{B}}

\newcommand{\I}{\mathbf{I}}

\newcommand{\flag}{\mathcal{FL}}

\newcommand{\cmark}{\ding{51}}%
\newcommand{\xmark}{\ding{55}}%

\newtheorem{remark}{Remark}

\newtheorem{prop}{Proposition}
\newtheorem{dfn}{Definition}
\newtheorem{exmp}{Example}[section]

\renewcommand{\paragraph}[1]{{\vspace{1mm}\noindent \bf #1}.}

\DeclareMathOperator*{\argmin}{arg\min}

\providecommand{\openbox}{\leavevmode
  \hbox to.77778em{%
  \hfil\vrule
  \vbox to.675em{\hrule width.6em\vfil\hrule}%
  \vrule\hfil}}
\makeatletter
\DeclareRobustCommand{\qed}{%
  \ifmmode
    \eqno \def\@badmath{$$}
    \let\eqno\relax \let\leqno\relax \let\veqno\relax
    \hbox{\openbox}%
  \else
    \leavevmode\unskip\penalty9999 \hbox{}\nobreak\hfill
    \quad\hbox{\openbox}%
  \fi
}

\crefname{eq}{eq}{eq}
\Crefname{Eq}{Eq}{Eq}
\crefname{thm}{theorem}{theorem}
\Crefname{Thm}{Theorem}{Theorem}
\crefname{prop}{Prop.}{Prop.}
\crefname{dfn}{Dfn.}{Dfn.}
\Crefname{Prop}{Proposition}{Proposition}
\crefname{remark}{remark}{remark}
\Crefname{Remark}{Remark}{Remark}
\Crefname{algorithm}{Alg.}{Alg.}

\crefname{table}{Tab.}{Tab.}
\Crefname{table}{Tab.}{Tab.}

\newcommand{\algname}{Flag-BMGS}

\begin{document}

\maketitle

\begin{abstract}
Flag manifolds encode nested sequences of subspaces and serve as powerful structures for various computer vision and machine learning applications. Despite their utility in tasks such as dimensionality reduction, motion averaging, and subspace clustering, current applications are often restricted to extracting flags using common matrix decomposition methods like the singular value decomposition. Here, we address the need for a general algorithm to factorize and work with hierarchical datasets. In particular, we propose a novel, flag-based method that decomposes arbitrary hierarchical real-valued data into a hierarchy-preserving flag representation in Stiefel coordinates. Our work harnesses the potential of flag manifolds in applications including denoising, clustering, and few-shot learning.
\end{abstract}


\section{Introduction}\label{sec:intro}
Hierarchical structures are fundamental across a variety of fields: they shape taxonomies and societies~\cite{lane2006hierarchy}, allow us to study 3D objects~\cite{leng2024hypersdfusion}, underpin neural network architectures~\cite{yan2015hd}, and form the backbone of language~\cite{longacre1966hierarchy}. They reflect parts-to-whole relationships~\cite{taher2024representing} and how our world organizes itself compositionally~\cite{salthe1985evolving}.
However, when handling hierarchies in data, we often resort to the temptation to \textit{flatten} them for simplicity, losing essential structure and context in the process. This tendency is evident in standard dimensionality reduction techniques, like principal component analysis, which ignore any hierarchy the data contains.

In this work, we advocate for an approach rooted in flags to preserve the richness of hierarchical linear subspaces. A flag~\cite{Mankovich_2023_ICCV,mankovich2024fun} represents a sequence of nested subspaces with increasing dimensions, denoted by its type or signature $(n_1,n_2,\dots,n_k;n)$, where $n_1$$<$$n_2$$<$$\dots$$<$$n_k$$<$$n$. For instance, a flag of type $(1,2;3)$ describes a line within a plane in $\mathbb{R}^3$. By working in flag manifolds—structured spaces of such nested subspaces—we leverage the full complexity of hierarchical data. Flag manifolds have already shown promise in extending traditional methods like Principal Component Analysis (PCA)~\cite{pennec2018barycentric, mankovich2024fun, szwagier2024curseisotropyprincipalcomponents}, Independent Component Analysis (ICA)~\cite{nishimori2006riemannian,nishimori2006riemannian0,nishimori2006riemannian1,nishimori2007flag,nishimori2008natural}, generalizing subspace learning~\cite{szwagier2025nestedsubspacelearningflags}, and Self-Organizing Maps (SOM)~\cite{ma2022self}. They enable robust representations for diverse tasks: averaging motions~\cite{Mankovich_2023_ICCV}, modeling variations in face illumination~\cite{draper2014flag, Mankovich_2023_ICCV}, parameterizing 3D shape spaces~\cite{ciuclea2023shape}, and clustering subspaces for video and biological data~\cite{marrinan2014finding, mankovich2022flag, Mankovich_2023_ICCV, mankovich2023module}.

\begin{figure}[t!]
    \centering
    \includegraphics[width=\linewidth]{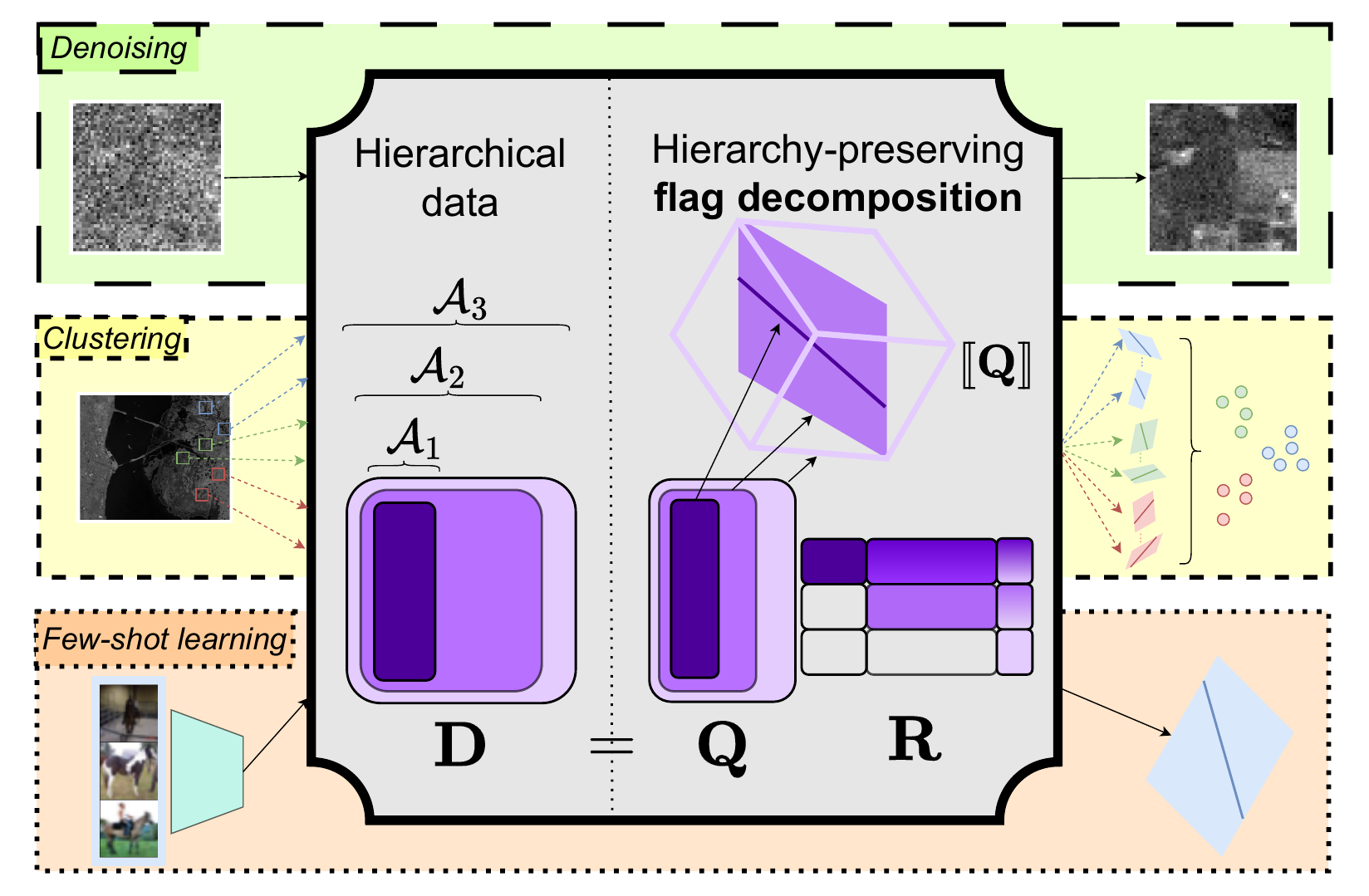}
    \caption{A flag decomposition (center) is used for a hierarchy-preserving flag representation and reconstruction. This decomposition separates a data matrix $\D$ with an associated hierarchy of column indices $\cA_1 \subset \cA_2 \subset \cA_3$ into Stiefel coordinates $\Q$ for a flag $[\![\Q]\!]$, a block-upper triangular matrix $\bR$, and a permutation matrix $\bP$ (not pictured). Example applications 
    include denoising (green), clustering (yellow), and few-shot learning (orange).\vspace{-3mm}}
    \label{fig:concept}
\end{figure}
Our main contribution is a Flag Decomposition (FD) specifically designed to preserve hierarchical structures within data (see~\cref{fig:concept}). First, we formalize the notion of hierarchies in data and the definition of a flag (\S\ref{sec:bg}). Next, we provide a practical algorithm for deriving FDs (\S\ref{sec:methods}) and outline multiple promising applications (\S\ref{sec:apps}). Then we demonstrate its robustness in clustering tasks involving noisy and outlier-contaminated simulated data (\S\ref{sec:results}). Our approach outperforms standard methods, such as Singular Value Decomposition (SVD), in clustering and denoising hyperspectral satellite images. Finally, we show that using flags as prototypes in a few-shot framework improves classification accuracy on benchmark datasets. Final remarks are in (\S\ref{sec:conclusion}) and formal proofs are in the suppl. material. Our implementation is \url{https://github.com/nmank/FD}.

\vspace{-1mm}
\vspace{-1mm}\section{Preliminaries}\vspace{-1mm}\label{sec:bg}
We begin by formalizing hierarchical datasets. Then, build to a definition of flags by providing the necessary background in matrix spaces. For notation, italicized capital letters (\eg, $\cA$) denote sets, and boldface letters denote matrices and column vectors (\eg, $\mathbf{X}$ and $\mathbf{x}_i$). $[\X]$ denotes the subspace spanned by the columns of $\X$.

Consider the data matrix $\D \in \R^{n \times p}$ with a hierarchy defined by the subsets of column indices
\begin{equation}
    \emptyset = \mathcal{A}_0 \subset \mathcal{A}_1 \subset \mathcal{A}_2 \subset \cdots \subset \mathcal{A}_k = \{1,2,\dots,p\}.
\end{equation}
Let $\D_{\cA_i}$ be the matrix containing only columns of $\D$ in $\mathcal{A}_i$.
\begin{dfn}[Column hierarchy for $\D$]\label{def:col_hierarchy}
    We call $\mathcal{A}_1 \subset \mathcal{A}_2 \subset \cdots \subset \mathcal{A}_k$ a column hierarchy for $\D \in \R^{n \times p}$ when
    \begin{equation}
        \mathrm{dim}([\D_{\cA_{i-1}}]) < \mathrm{dim}([\D_{\cA_i}]) \text{ for } i=1,2,\dots,k.
    \end{equation}
\end{dfn}

Given a column hierarchy for $\D$~\footnote{A similar, complex, and well-studied notion of hierarchical matrices is H-matrices~\cite{borm2003introduction}.}, there is a natural correspondence between column and subspace hierarchies
\begin{equation*}
    \begin{matrix}
        \text{columns:} & \cA_1 & \subset & \cA_2 & \subset & \cdots & \subset & \cA_k\\
        \text{subspaces:} & [\D_{\cA_1}] & \subset & [\D_{\cA_2}] & \subset & \cdots & \subset & [\D_{\cA_k}].
    \end{matrix}
\end{equation*}
These hierarchies can include coarse-to-fine neighborhoods (\eg,~\cref{ex:pixel_hierarchy}), spectral hierarchies (\eg,~\cref{ex:spectrum_hierarchy}), and feature representations (\eg,~\cref{ex:feature_hierarchy}).
\begin{exmp}[Neighborhood Hierarchy]\label{ex:pixel_hierarchy}
    Consider $(p_i\times p_i)$ concentric RGB image patches increasing in size with $i=1,2,3$. We store the $i^{\mathrm{th}}$ image patch in $\D \in \R^{3 \times p_i^2}$. $\cA_1$ contains the column indices of the smallest image patch in $\D$, $\cA_2$ contains those of the next smallest patch, and $\cA_3$ the largest patch. This results in the neighborhood column hierarchy $\cA_1 \subset \cA_2 \subset \cA_3$ for the data matrix $\D$.    
\end{exmp}

\begin{exmp}[Spectral Hierarchy]\label{ex:spectrum_hierarchy}
    Let $\D \in \R^{n \times p}$ be a hyperspectral image with $n$ pixels and $p$ bands. A hierarchy is imposed on the bands by grouping wavelengths into: 
    \begin{figure}[H]
        \vspace{-4mm}
        \centering
        \includegraphics[width=\linewidth, trim = 7mm 0mm 4mm 0mm, clip]{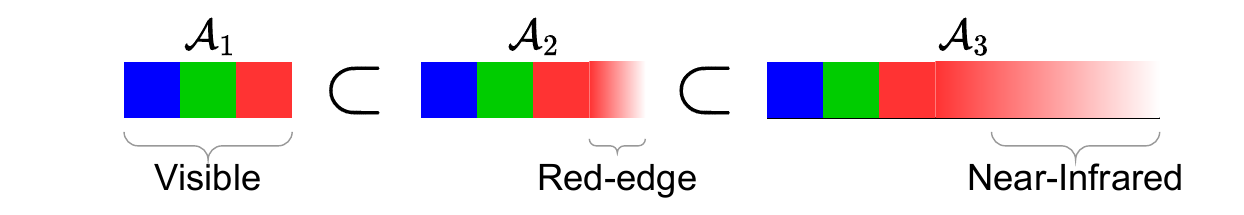}
        \vspace{-8mm}
    \end{figure}
\end{exmp}

\begin{exmp}[Feature Hierarchy]\label{ex:feature_hierarchy}
    Consider a feature extractor (\eg, deep network) admitting the following decomposition: $f_\Theta = f^{(2)}_\Theta \circ f^{(1)}_\Theta: \R^N \rightarrow \R^{n}$ where $f^{(1)}_\Theta:\R^{N} \rightarrow \R^{n}$. $s$ samples $\bm{x}_1,\cdots,\bm{x}_s$ are used to obtain 
    \begin{equation*}
         \mathbf{D} = \left[ f^{(1)}_\Theta(\bm{x}_1)\,|\, \cdots\,|\, f^{(1)}_\Theta(\bm{x}_s)\,|\,f_\Theta(\bm{x}_1)\,|\, \cdots| f_\Theta(\bm{x}_s) \right] .
    \end{equation*}
    Since information flows from $f^{(1)}_\Theta$ to  $f^{(2)}_\Theta$, it is natural to assume that features extracted by $f^{(1)}_\Theta$ span a subspace of the features extracted by $f_{\Theta}$. Therefore, we propose the hierarchy $\{ 1,2,\dots, s\} \subset \{ 1,2,\dots,2s\}$.
\end{exmp}

\noindent Next, we build a mathematical formalization of flags.

\begin{dfn}[Matrix spaces] The \textbf{orthogonal group} $O(n)$ denotes the group of distance-preserving transformations of a Euclidean space of dimension $n$:
\begin{align}
    O(n) := \{ \M \in \R^{n \times n} : \M^\top \M =\M\M^\top = \I\}.
\end{align}
A \textbf{permutation matrix} is a square matrix $\bP \in \R^{n \times n}$ where each column and each row contains only a single $1$ value and the rest are $0$. $\D\bP$ permutes the columns of $\D$.
An important property of permutation matrices is $\bP^{-1} = \bP^\top$.
\noindent The \textbf{Stiefel manifold} $St(k,n)$, a.k.a. the set of all orthonormal $k$-frames in $\R^n$, can be represented as the quotient: $St(k,n) = O(n)/O(n-k)$. A point on the Stiefel manifold is parameterized by a tall-skinny $n \times k$ real matrix with orthonormal columns, \ie $\X \in \R^{n \times k}$ where $\X^\top \X = \I$.
\noindent The \textbf{Grassmannian}, $ Gr(k, n) $ represents the set of all $k $-dimensional subspaces of $ \mathbb{R}^n $. Each point in $ Gr(k, n) $ can be identified with an equivalence class of matrices in the Stiefel manifold, where two matrices are equivalent if their columns span the same subspace. 
We represent $[\X] \in Gr(k,n)$ using the Stiefel coordinates
$\X \in St(k,n)$.
\end{dfn}

\begin{dfn}[Flag]
    Let $\mathcal{V}_i$ be an $n_i$-dimensional subspace of a vector space $\mathcal{V}$ of dimension $n$. A \emph{flag} of type $(n_1,n_2,\dots, n_k;n)$, is the nested sequence of subspaces
    \begin{equation}
        \emptyset \subset \mathcal{V}_1 \subset \mathcal{V}_2 \subset \cdots \subset \mathcal{V}_k \subset \mathcal{V}.
    \end{equation}
\end{dfn}
The \emph{flag manifold} $\flag(n_1,n_2,\dots, n_k;n)$ is a Riemannian manifold and the collection of all flags of type, a.k.a. signature, $(n_1,n_2,\dots, n_k;n)$~\cite{szwagier2023rethinking,ye2022optimization}. The first empty subspace, $\emptyset$ with dimension $n_0 = 0$, is now mentioned for completeness but will be dropped from notation and implicit from here on. In this work, we will work with real flags, namely $\mathcal{V} = \mathbb{R}^n$.

\begin{remark}[Flag manifold as a quotient of groups]
Ye et al.~\cite[Proposition 4.10]{ye2022optimization} prove that $\flag(n_1,\dots,n_k;n)$ is diffeomorphic to the quotient space $St(n_k,n)/(O(m_1) \times O(m_2) \times \cdots \times O(m_{k}))$ where $m_i = n_i - n_{i-1}$. This fact gives a Stiefel manifold coordinate representation of a flag.
\end{remark}

\begin{dfn}[Stiefel coordinates for flags~\cite{ye2022optimization}]
A flag is represented by $\X  = [ \X_1 | \X_2 | \cdots | \X_k ] \in St(n_k,n)$ where $\X_i \in \R^{n \times m_i}$. Specifically, $\X$ represents the flag 
\begin{equation}
    [\![ \X ]\!] = [\X_1] \subset [\X_1, \X_2] \subset \cdots \subset [\X_1, \X_2,\dots \X_k] \subset \R^n
\end{equation}
We say $[\![ \X ]\!]$ is a flag of type (or signature) $(n_1,n_2,\dots, n_k;n)$ and $[\X_1, \X_2,\dots \X_i]$ denotes the span of $[\X_1| \X_2 |\cdots |\X_i]$ (for $i=1,2,\dots,k$).
\end{dfn}

\begin{table*}[t]
    \centering
    \caption{Computing the chordal distance on Steifel, Grassmann, and flag manifolds using matrix representatives. $\|\cdot\|_F$ is Frobenius norm.}
    \setlength{\tabcolsep}{4mm} 
    {%
    \begin{tabular}{l| c | c |c}
        Representation / Manifold & $\X,\Y \in St(n_k,n)$ & $[\X],[\Y] \in Gr(n_k,n)$  &  $[\![\X]\!],[\![\Y]\!] \in \flag(1,\dots,n_k;n)$\\
        \midrule
        Chordal distance & $\|\X - \Y\|_F$ &  $\frac{1}{\sqrt{2}}\|\X\X^\top - \Y\Y^\top\|_F$ & $\sqrt{\frac{1}{2}\sum_{i=1}^k\|\X_i\X_i^\top - \Y_i\Y_i^\top\|_F^2}$
    \end{tabular}
    }
    \label{tab:chordaldist}
\end{table*}

Given the tall-skinny orthonormal matrix representatives $\X,\Y \in \R^{n \times n_k}$, we also utilize their \emph{chordal distances} as given in~\cref{tab:chordaldist}. 
The chordal distance on the Stiefel manifold measures the $2$-norm of the vectorized matrices. In contrast, the Grassmannian chordal distance measures the $2$-norm of the vector of sines of the principal angles~\cite{bjorck1973numerical} between the subspaces through the Frobenius norm of the projection matrices~\cite{edelman1998geometry}. The chordal distance on the flag manifold~\cite{pitaval2013flag} arises from the fact that it is a closed submanifold of $\Gr(m_1,n) \times \cdots \times \Gr(m_k,n)$ as shown by Ye et al.~\cite[Proposition 3.2]{ye2022optimization}. This distance is similar to the Grassmannian chordal distance between subsequent pieces of the flags (\eg, $[\X_i]$ and $[\Y_i]$ for $i=1,\dots,k$).

\section{Flag Decomposition (FD)}\label{sec:methods}
We will now introduce our novel Flag Decomposition (FD) that, given $\D$, outputs a hierarchy-preserving flag $[\![\Q]\!]$. From this point on, $[\cdot,\cdot,\cdot]$ denotes column space and $[\cdot|\cdot|\cdot]$ block matrices. Let $\cB_i = \cA_i \setminus \cA_{i-1}$ be the difference of sets for $i=1,2,\dots, k$ and $\B_i = \D_{\cB_i}$ so that $[\D_{\cA_i}] = [\B_1,\B_2,\dots,\B_i]$. We define the permutation matrix $\bP$ so $\B = [\B_1|\B_2| \cdots | \B_k] = \D\bP$. Also, denote the projection matrix onto the null space of $[\Q_i]$ with $\Q_i \in St(m_i,n)$ as $\boldsymbol{\Pi}_{\Q_i^\perp} = \I - \Q_i\Q_i^\top$. We use $n_0 = 0$, $\cA_0 = \emptyset$, and $\boldsymbol{\Pi}_{\Q_0^\perp} = \I$.

\begin{dfn}[Hierarchy-preserving flags]
     A flag $[\![\X]\!] \in \flag(n_1,n_2,\dots, n_k;n)$ is said to preserve the hierarchy of $\D$ if $[\D_{\cA_i}] = [\X_1,\X_2,\dots,\X_i]\,$ for each $i=1,2,\dots,k$. 
\end{dfn}

If $\mathcal{A}_1 \subset \mathcal{A}_2 \subset \cdots \subset \mathcal{A}_k$ is a column hierarchy for $\D$, then a hierarchy-preserving flag results in the three, equivalent, nested sequences of subspaces
\begin{equation*}
    \begin{matrix}
        [\D_{\cA_1}] & \subset & [\D_{\cA_2}] & \subset & \cdots & \subset & [\D_{\cA_k}]\\    
        [\B_1] & \subset & [\B_1,\B_2] & \subset & \cdots & \subset & [\B_1,\B_2, \dots,\B_k]\\
        [\X_1] & \subset & [\X_1,\X_2] & \subset & \cdots & \subset & [\X_1,\X_2,\dots,\X_k].
    \end{matrix}
\end{equation*}
SVD and QR decomposition can recover flags from data with certain, limited column hierarchies (see suppl. material for details). However, when faced with a more complex column hierarchy, both QR and SVD cannot recover the entire hierarchy-preserving flag (see~\cref{fig:flag cartoon1}).

These examples motivate a generalized decomposition of $\D$ that outputs a hierarchy-preserving flag. In particular, unlike in QR decomposition, $\D$ can be rank-deficient (\eg, $\mathrm{rank}(\D) < p$); and unlike the SVD, we can decompose into flags of type $(n_1,n_2,\dots,n_k;n)$ with $n_k \leq p$. 

\begin{figure}[t]
    \centering
    \includegraphics[width=.95\linewidth]{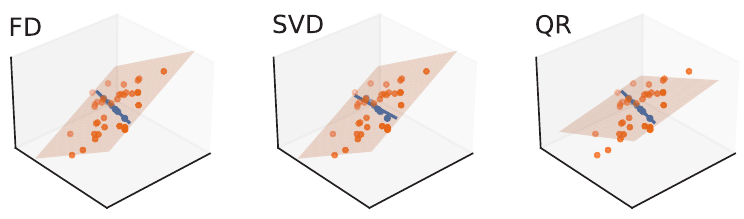}
    \vspace{-1mm}
    \caption{We recover a flag from $\D$ with hierarchy $\cA_1 \subset \cA_2$. Columns of $\D$ are plotted as points with $\cA_1$ in blue and $\cA_2 \setminus \cA_1$ in orange. FD is the only method that recovers the flag (line inside plane). SVD correctly recovers the plane but not the line whereas QR only recovers the line and the plane misses the orange points.}\vspace{-4mm}
    \label{fig:flag cartoon1}
\end{figure}
\begin{dfn}[\textbf{Flag Decomposition (FD)}]
    Let $\D \in \R^{n \times p}$ be data with the hierarchically nested sequence of column indices $\cA_1 \subset \cA_2 \subset \cdots \subset \cA_k$. A flag decomposition of type $(n_1,n_2, \cdots, n_k; n)$ is the matrix factorization
    \begin{equation}
        \D = \Q \bR \bP^\top
    \end{equation}
    where the block structures are
    \begin{align}
        \Q &= [\underbrace{\Q_1}_{n \times m_1} | \underbrace{\Q_2}_{n \times m_2} | \cdots | \underbrace{\Q_k}_{n \times m_k} ] \in \R^{n \times n_k},\\
        \bR &= \begin{bmatrix}
            \bR_{11} & \bR_{12} & \cdots & \bR_{1k}\\
            \mathbf{0} & \bR_{22} & \cdots & \bR_{2k}\\
            \vdots & \vdots & \ddots & \vdots\\
            \mathbf{0} & \mathbf{0} & \cdots & \bR_{kk} \end{bmatrix} \in \R^{n_k \times p}, \\
        \bP &= [\bP_1 \,|\, \bP_2 \,|\, \cdots \,|\, \bP_k] \in \R^{p \times p}.
    \end{align}
    Here, $\Q$ corresponds to the Stiefel coordinates for the hierarchy-preserving flag $[\![\Q]\!] \in \flag(n_1,n_2,\dots, n_k;n)$ with $m_i=n_i - n_{i-1}$ and $n_k \leq p$, $\bR$ is a block upper-triangular matrix, and $\bP$ is a permutation matrix so that $\Q \bR = \D \bP$.
\end{dfn}

We now use~\cref{prop:stiefel_coords} to determine when we can recover a hierarchy-preserving flag from data and then we use~\cref{prop:proj_prop_and_flags} to show how to construct $\bR$ and $\bP$ from this flag. Finally, we combine~\cref{prop:stiefel_coords,prop:proj_prop_and_flags} to define when we can find an FD (see~\cref{prop:fd_def}) and investigate its uniqueness (see~\cref{prop:block_amb}).

\begin{prop}\label{prop:stiefel_coords}
    Suppose $\cA_1 \subset \cA_2 \subset \cdots \subset \cA_k$ is a column hierarchy for $\D$. Then there exists $\Q= [\Q_1\,|\,\Q_2\,|\, \cdots\,|\,\Q_k]$ that are coordinates for the flag $[\![\Q]\!]\in\flag(n_1,n_2,\dots,n_k;n)$ where $n_i = \mathrm{rank}(\D_{\mathcal{A}_i})$ that satisfies $[\Q_i] = [\boldsymbol{\Pi}_{\Q_{i-1}^\perp}\cdots \boldsymbol{\Pi}_{\Q_1^\perp}\B_i]$ and the \textbf{projection property} (for $i=1,2\dots,k$): 
    \begin{equation}\label{eq:projection_property}
    \boldsymbol{\Pi}_{\Q_i^\perp}\boldsymbol{\Pi}_{\Q_{i-1}^\perp}\cdots \boldsymbol{\Pi}_{\Q_1^\perp} \B_i = 0.
    \end{equation}
\end{prop}
\begin{proof}
    Define (for $i=2,3,\dots,k$) the projector onto the null space of $[\Q_1,\Q_2,\dots,\Q_i]$, as $\boldsymbol{\Pi}_{\Q_{:i}^\perp} = \I - \Q_{:i}\Q_{:i}^\top$. We use this to define $\C_i = \boldsymbol{\Pi}_{\Q_{:i-1}^\perp}\B_i$ and $\Q_i \in St(m_i,n)$ so that $[\Q_i] = [\C_i]$. Then we use mathematical induction to show results ending in ~\cref{eq:projection_property} and $\Q_i \in St(m_i,n)$ with $m_i = n_i - n_{i-1}$ where $n_i = \mathrm{rank}(\D_{\cA_i})$.
\end{proof}

The simplest methods for recovering $\Q$ so that $[\Q_i] = [\boldsymbol{\Pi}_{\Q_{i-1}^\perp}\cdots \boldsymbol{\Pi}_{\Q_1^\perp}\B_i]$ include the left singular vectors from the truncated SVD and the $\Q$ matrix from the QR decomposition with pivoting. We will now use $\Q$ and the projection property to construct $\bR$ and $\bP$ for the FD.


\begin{prop}\label{prop:proj_prop_and_flags}
 Suppose $\cA_1 \subset \cA_2 \subset \cdots \subset \cA_k$ is a column hierarchy for $\D$. Then there exists some hierarchy-preserving 
 $[\![\Q]\!] \in \flag(n_1,n_2,\dots, n_k;n)$ (with $n_i = \mathrm{rank}(\D_{\mathcal{A}_i})$) 
 that satisfies the projection property of $\D$ and can be used for a flag decomposition of $\D$ with
    \begin{align}
        \bR_{i,j} &= 
        \begin{cases}
            \Q_i^\top\boldsymbol{\Pi}_{\Q_{i-1}^\perp}\cdots \boldsymbol{\Pi}_{\Q_1^\perp} \B_i, &i=j\\
            \Q_i^\top\boldsymbol{\Pi}_{\Q_{i-1}^\perp}\cdots \boldsymbol{\Pi}_{\Q_1^\perp} \B_j, &i < j
        \end{cases},\label{eq:onlyR}\\
        \bP_i &= \left[ \,\mathbf{e}_{b_{i,1}}\,|\, \mathbf{e}_{b_{i,2}}\,|\, \cdots\,|\, \mathbf{e}_{b_{i,|\cB_i|}} \right]\label{eq:onlyP}
    \end{align}
    where $\{b_{i,j}\}_{j=1}^{|\cB_i|} = \cB_i$ and $\mathbf{e}_{b}$ is the $b_{i,j}$$^{\mathrm{th}}$ standard basis vector.
\end{prop}
\begin{proof}[Proof sketch]
    This is proved using the previous proposition.
\end{proof}

\begin{prop}\label{prop:fd_def}
    A data matrix $\D$ admits a flag decomposition of type $(n_1,n_2, \cdots, n_k; n)$ if and only if $\cA_1 \subset \cA_2 \subset \cdots \subset \cA_k$ is a column hierarchy for $\D$.
\end{prop}
\begin{proof}[Proof sketch]
    We use~\cref{prop:stiefel_coords,prop:proj_prop_and_flags} and the definition of a column hierarchy for $\D$. Details in suppl. material.
\end{proof}
\emph{Therefore, any $\D$ with an associated column hierarchy admits a hierarchy-preserving FD.} Now we state a uniqueness result for the FD.

\begin{prop}[Block rotational ambiguity]\label{prop:block_amb}
    Given the FD $\D = \Q \bR \bP^\top$, any other Stiefel coordinates for the flag $[\![\Q]\!]$ produce an FD of $\D$ (via~\cref{prop:proj_prop_and_flags}). Furthermore, different Stiefel coordinates for $[\![\Q]\!]$ produce the same objective function values in~\cref{eq:general_opt} and~\cref{eq:iterative_opt} (for $i=1,\cdots,k$).
\end{prop}
\begin{proof}[Proof sketch]
    Notice $\Q_i \Q_i^T = (\Q_i\mathbf{M}_i)  (\Q_i\mathbf{M}_i)^\top$ for any $\Q_i \in St(m_i,n)$ and $\mathbf{M}_i \in O(m_i)$. See our suppl. material for details.
\end{proof}

\subsection{Flag recovery}
In this section, we introduce an approach for recovering the FD $\D = \Q \bR \bP^\top$ from a given, corrupted version of the dataset, $\tilde{\D}$ and the column hierarchy $\cA_1 \subset \cA_2 \subset \cdots \subset \cA_k$ for $\D$. We call recovering $[\![ \Q ]\!]$ from $\tilde{\D}$ and $\cA_1 \subset \cA_2 \subset \cdots \subset \cA_k$ the \emph{flag recovery}.

Recall that any $[\![ \Q ]\!]$ satisfying the projection property of $\D$ can be used for a FD (see~\cref{prop:proj_prop_and_flags}). However, since we only have access to $\tilde{\D}$, we may not be able to satisfy this property. As a remedy, we try to get as close as possible to satisfying the projection property by optimizing for $[\![\Q]\!]$ such that $\boldsymbol{\Pi}_{\Q_i^\perp} \cdots \boldsymbol{\Pi}_{\Q_1^\perp}\tilde{\B}_i \approx \bm{0}$ for each $i=1,2,\dots,k$. We minimize this cost column-wise to solve the problem in maximum generality. Specifically, we propose the following minimization:
\begin{equation}\label{eq:general_opt}
    [\![\Q]\!] = \argmin_{[\![\X]\!] \in \flag(n_1,n_2, \dots, n_k;n)} \sum_{i=1}^k \sum_{j \in \cB_i}\| \boldsymbol{\Pi}_{\X_i^\perp} \cdots \boldsymbol{\Pi}_{\X_1^\perp} \tilde{\mathbf{d}}_j \|_r^q
\end{equation}
for $r\geq 0$, $q > 0$. Choosing small $r$ and $q$ (\eg, $r =0$ and $q=1$) would result in a robust flag recovery, optimal for recovering $\D$ in the presence of outlier columns in $\tilde{\D}$. This problem is difficult, even after restricting $q$ and $r$, so we address the iterative optimization for each $\Q_i$ for $i=1$, then $i=2$, and so on until $i=k$.
\begin{equation}\label{eq:iterative_opt}
    \Q_i =  \argmin_{\X \in St(m_i,n)} \sum_{j \in \cB_j}\| \boldsymbol{\Pi}_{\X^\perp} \boldsymbol{\Pi}_{\Q^\perp_{i-1}} \cdots \boldsymbol{\Pi}_{\Q^\perp_1} \tilde{\mathbf{d}}_j \|_r^q.
\end{equation}
The solution to the case where $r=q=2$ is obtained by the first $m_i$ left singular vectors of $\boldsymbol{\Pi}_{\Q^\perp_{i-1}} \cdots \boldsymbol{\Pi}_{\Q^\perp_1} \tilde{\D}_{\mathcal{B}_j}$.
In general, solving~\cref{eq:iterative_opt} for some $i$ recovers $\Q_i$ whose columns form a basis for a $m_i$ dimensional subspace in $\R^n$.
Although outputting a truncated basis via QR with pivoting or rank-revealing QR decompositions would offer faster alternatives to SVD for solving~\cref{eq:iterative_opt}, SVD offers more reliable subspace recovery~\cite{demmel1997applied}. Thus, we use SVD-based algorithms and leave QR methods for future work.

For cases where $\tilde{\D}$ has outlier columns, we use an $L_1$ penalty, \ie, $q=1$, and introduce an \textbf{IRLS-SVD solver}\footnote{IRLS denotes iteratively reweighted least squares.}, a simple method that resembles IRLS algorithms for subspace recovery~\cite{zhang2014novel,lerman2015robust,vidal2018dpcp,lerman2018fast,lerman2018overview,garg2019subspace,mankovich2022flag}. In practice, we implement a vanilla IRLS-SVD algorithm which could further be made faster and provably convergent using tools from~\cite{aftab2014generalized,beck2015weiszfeld,kummerle2021iteratively,kummerle2021scalable,verdun2024fast}. We leave more advanced solvers, as well as working with other values of $r$ and $q$ (e.g., $r=0$~\cite{liu2012robust}), for future work.

\subsection{\algname}
We now propose~\algname, an algorithm for finding FD and its robust version, Robust FD (RFD). Our algorithm is inspired by the Block Modified Gram-Schmidt (BMGS) procedure~\cite{jalby1991stability,barlow2019block}. Modified Gram-Schmidt (MGS) is a more numerically stable implementation of the classical Gram-Schmidt orthogonalization. BMGS runs an MGS algorithm on block matrices, iteratively projecting and ortho-normalizing matrices rather than vectors, to output a QR decomposition. In contrast, we use~\algname~on a data matrix with a column hierarchy to produce a hierarchy-preserving FD. We summarize the properties of Gram-Schmidt variants in~\cref{tab:alg_table1}. 
\begin{table}[t!]
    \centering
    \caption{A summary of GS algorithms and their properties.}
    \setlength{\tabcolsep}{8pt}
    \label{tab:alg_table1}
    \resizebox{\columnwidth}{!}{
    \begin{tabular}{c|cccc}
        \toprule
        Algorithm & GS & MGS & BMGS & \algname \\
        \midrule
        Stable & \xmark & \cmark & \cmark & \cmark \\ 
        Block-wise & \xmark & \xmark & \cmark & \cmark \\ 
        Hier.-pres. & \xmark & \xmark & \xmark & \cmark \\ 
        \bottomrule
    \end{tabular}  
    }
    \vspace{-4mm}
\end{table}

\algname~operates by first generating a permutation matrix $\bP$ (see~\cref{eq:onlyP}) to extract the matrix $\B = \D\bP^\top$, using the column hierarchy. Then each iteration $i=1,2,\dots,k$ constructs $\boldsymbol{\Pi}_{\Q_{i-1}^\perp}\cdots \boldsymbol{\Pi}_{\Q_1^\perp} \B_i$, solves an optimization of the form~\cref{eq:iterative_opt}, and then constructs each $\bR_{i,j}$ for $j \leq i$ (see~\cref{eq:onlyR}). In experiments, we call FD the output of~\algname~using SVD with $r=q=2$ and Robust FD (RFD) the iterative variant using $r=2$ and $q=1$ to solve~\cref{eq:iterative_opt}. Algorithm details are in suppl. material.

Stability results and the search for more optimal algorithms, such as those using block Householder transformations~\cite{griem2024block} are left to future work. Many other block matrix decompositions exist and a brief discussion of such a low-rank block matrix decomposition~\cite{ong2016beyond} can be found in suppl. material.

\paragraph{On the flag type}
Flag type is an input to~\algname. \emph{Detecting} or selecting an adapted flag type from data rather than relying on a heuristic choice, is recently addressed by Szwagier~\etal in principal subspace analysis~\cite{szwagier2024curseisotropyprincipalcomponents}. 
The FD model does not benefit from this advance because it preserves hierarchies rather than directions of maximum variance. We now discuss methods for estimating flag type.


Assuming full access to $\D$, the flag type is $(n_1,n_2,\dots,n_k;n)$ where $n_i = \mathrm{rank}(\D_{\cA_i})$ (see~\cref{prop:stiefel_coords}). 
Yet, the data can be corrupted, \ie, we observe only $\tilde{\mathbf{D}} = \mathbf{D} + \bm{\epsilon}$ ($ \bm{\epsilon} $ denotes random noise) instead of the true $\D$. This leads to an estimation problem of the flag type of the FD assuming access to $\tilde{\D}$ and the true (known) column hierarchy for $\D$.     

A naive approach to address the problem of flag type estimation for our FD is to run the FD along with a singular value truncation in each SVD. Methods for truncating the SVs include the \emph{elbow} and \emph{Gavish-Dohono}~\cite{gavish2014optimal,falini2022review}. In this work, given a column hierarchy and $\tilde{\D}$ (but not $\D$), we choose a flag type where $n_k < \mathrm{rank}(\tilde{\D})$ and input it to FD. In doing so, the output of FD forms a reduced-rank approximation of $\D$ denoted $\hat{\D} = \Q \bR \bP^\top$. 

A promising future research direction involves exploring smarter truncation methods for extracting the flag type of $\D$ under specific contamination criteria.

\section{Applications}\label{sec:apps}
Before moving on to experimental results, we specify applications of Flag Decomposition (FD), which enables reconstruction in the presence of data corruption, visualization, classification, and a novel prototype and distance for few-shot learning. 

\subsection{Reconstruction}
Consider the matrix $\D$ with an associated column hierarchy $\cA_1 \subset \cA_2 \subset \cdots\subset  \cA_k$. Suppose we have a corrupted version $\tilde{\mathbf{D}}$ and the feature hierarchy is known a priori. We use FD to recover $\D$ from $\tilde{\D}$. For example, $\tilde{\D}$ could be a (pixels $\times$ bands) flattened hyperspectral image with a hierarchy on the bands (see~\cref{ex:spectrum_hierarchy} and~\cref{fig:concept}) with outlier bands or additive noise. Another example includes a hierarchy of subjects: with images of subject $1$ inside those of subjects $1$ \& $2$ (see~\cref{fig:flag cartoon2}). A reconstruction using FD respects this hierarchy by preserving the identities of the two subjects.
\begin{figure}[t!]
    \centering
    \vspace{-4mm}
    \includegraphics[width=\linewidth]{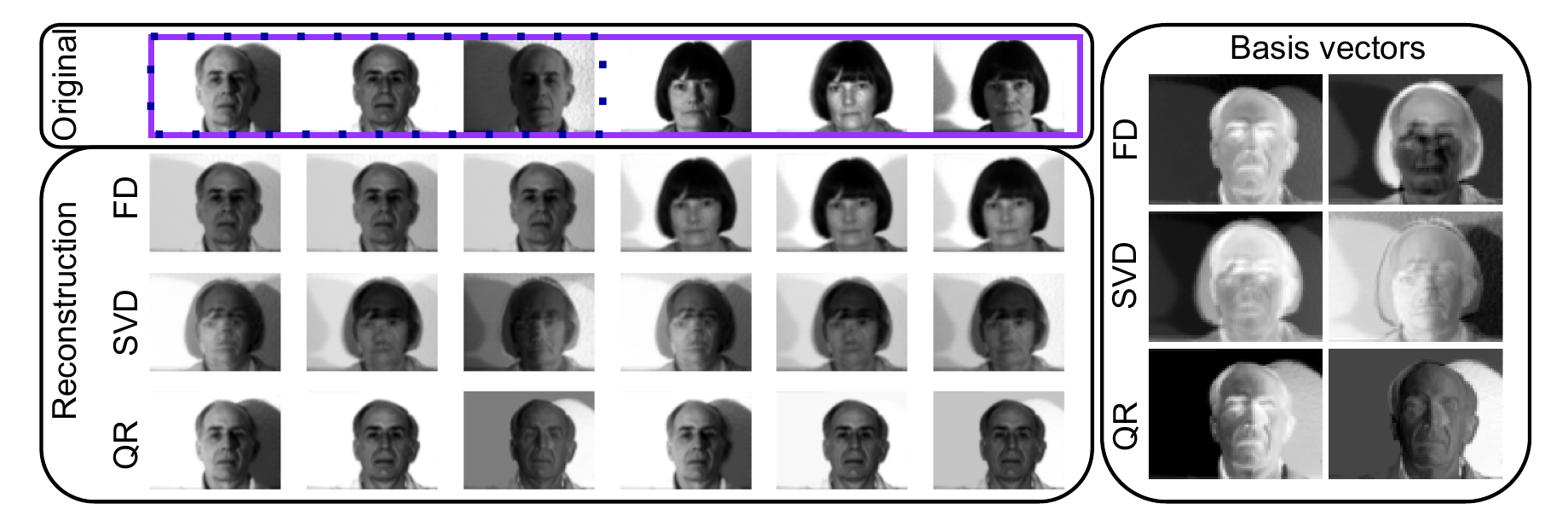}
    \caption{Images from the YFB~\cite{belhumeur1997eigenfaces} are flattened and horizontally stacked into $\D$. We use the hierarchy with the images of the first subject as $\cA_1$ and all images as $\cA_2$. We run FD (flag type $(1,2)$) and baselines (rank $2$). FD is the only method to correctly reconstruct the subjects. We plot the basis vectors (eigenfaces) on the right and find FD extracts basis elements that most closely resemble the subjects.}
    \vspace{-6mm}
    \label{fig:flag cartoon2}
\end{figure}

Suppose FD is computed by running~\algname~ on $\tilde{\D}$ to output $\Q, \bR, \bP$. Then we reconstruct $\hat{\D} = \Q \bR \bP^\top \approx \D$. This is a \textit{low-rank} reconstruction of $\tilde{\D}$ when $n_k < \mathrm{rank}(\tilde{\D})$. \emph{Unlike other reconstruction algorithms this application preserves the column hierarchy.}



\subsection{Leveraging the geometry of flags}
Consider a collection of $\mathcal{D} = \{\D^{(j)} \in \R^{n \times p_j}\}_{j=1}^N$ with column hierarchies $\cA_1^{(j)} \subset \cA_2^{(j)} \subset \cdots \subset \cA_k^{(j)}$. 
For example, $\mathcal{D}$ could be a collection of (band  $\times$ pixel) hyperspectral image patches. After choosing one flag type $(n_1,n_2,\dots,n_k;n)$ with $n_k \leq \mathrm{min}(p_1,\dots,p_N)$, we can use~\algname~with this flag type on each $\D^{(j)}$ to extract the collection of flags $ \mathcal{Q} = \{[\![\Q^{(j)}]\!]\}_{j=1}^N \subset \flag(n_1,n_2,\dots,n_k;n)$. Now, we can use chordal distance on the flag manifold $\flag(n_1,n_2,\dots,n_k;n)$ or the product of Grassmannians $Gr(m_1,n) \times Gr(m_2,n) \times \cdots \times Gr(m_k,n)$ to build an $N \times N$ distance matrix and run multidimensional scaling (MDS)~\cite{kruskal1978multidimensional} to visualize $\mathcal{D}$ or $k$-nearest neighbors~\cite{marrinan2021minimum} to cluster $\mathcal{D}$ (see~\cref{fig:concept}). Other clustering algorithms like $k$-means can also be implemented with means on products of Grassmannians (\eg,~\cite{fletcher2009geometric, draper2014flag, mankovich2022flag, mankovich2023subspace}) or chordal flag averages (\eg,~\cite{Mankovich_2023_ICCV}). Additionally, we can generate intermediate flags by sampling along geodesics between flags in $\mathcal{Q}$ using tools like {\tt manopt}~\cite{boumal2014manopt,townsend2016pymanopt} for exploration of the flag manifold between data samples.

\subsection{Few-shot learning}
In few-shot learning, a model is trained on very few labeled examples, a.k.a. `shots' from each class to make accurate predictions. 
Suppose we have a pre-trained feature extractor $f_\Theta: \mathcal{X} \rightarrow \R^{n}$, parameterized by $\Theta$. 
In few-shot learning, the number of classes in the training set is referred to as `ways.' We denote the feature representation of $s$ shots in class $c$ as $f_\Theta(\bm{x}_{c,1}),f_\Theta(\bm{x}_{c,2}),\dots, f_\Theta(\bm{x}_{c,s})$ where $\bm{x}_{c,i} \in \mathcal{X}$. The `support' set is the set of all shots from all classes. A few-shot learning architecture contains a method for determining a class representative (a.k.a. `prototype') in the feature space ($\R^n$) for each class using its shots. A test sample (`query') is then passed through the encoder, and the class of the nearest prototype determines its class. Overall, a classical few-shot learning architecture is comprised of three (differentiable) pieces: (1) a mapping of shots in the feature space to prototypes, (2) a measure of distance between prototypes and queries, and (3) a loss function for fine-tuning the pre-trained encoder.

\begin{figure}[t]
        \centering
        \includegraphics[width=\linewidth, trim= 38mm 0mm  38mm 0mm, clip]{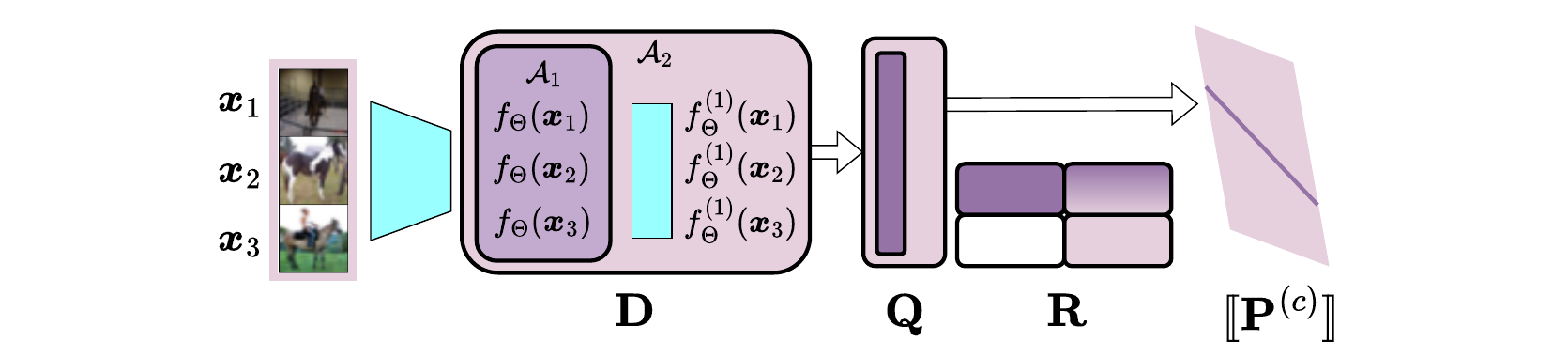}
        \caption{To perform few shot-learning, we embed all $s=3$ shots ($\bm{x}_1, \bm{x}_2, \bm{x}_3$) of one class into one flag using FD. The light shapes are pre-trained and frozen feature extractors.}
        \label{fig:fewshot_flags}
    \end{figure}
\paragraph{Flag classifiers}~\label{sec:fewshot_flags} 
Take a feature extractor that decomposes into $k \geq 2$ functions. Specifically, $f_{\Theta} = f_{\Theta}^{(k)} \circ \cdots \circ f_{\Theta}^{(1)}:\mathcal{X} \rightarrow \R^n$,
where each $f_{\Theta}^{(i)}$ maps to $\R^n$.
We can generalize~\cref{ex:feature_hierarchy} to construct a $k$-part hierarchical data matrix. For simplicity, we consider the case where $k =2$. After constructing a data matrix with a corresponding column hierarchy, we use~\algname~to represent the support of one class, $c$, in the feature space as $[\![\Q^{(c)}]\!]\in \flag(n_1,n_2; n)$ (see~\cref{fig:fewshot_flags}). Now, each subspace $[\Q^{(c)}_1]$ and $[\Q^{(c)}_2]$ represents the features extracted by $f_{\Theta}^{(1)}$ and $f_{\Theta}$, respectively.

    Given a flag-prototype $[\![\Q^{(c)}]\!] \in \flag(n_1,n_2;n)$ and query $\left\{f^{(1)}_{\Theta}(\bm{x}), f_{\Theta}(\bm{x})\right\} \subset\R^n$, we measure distance between the them as
    \begin{equation}\label{eq: fewshot_dist}
        \left \| \boldsymbol{\Pi}_{{\Q^{(c)}_1}^\perp}f_{\Theta}^{(1)}(\bm{x}) \right \|_2^2 + \left \| \boldsymbol{\Pi}_{{\Q^{(c)}_2}^\perp}f_{\Theta}(\bm{x}) \right \|_2^2
    \end{equation}
    where $\boldsymbol{\Pi}_{{\Q^{(c)}_i}^\perp} = \I - \Q^{(c)}_i {\Q^{(c)}_i}^\top$ for $i=1,2$.
    This is proportional to the squared chordal distance on $\flag(1,2;n)$ when the matrix $\left[ f_{\Theta}^{(1)}(\bm{x})| f_{\Theta}(\bm{x}) \right]$ is in Stiefel coordinates.

    Flag classifiers are \emph{fully differentiable} enabling fine-tuning of the feature extractor with a flag classifier loss. We leave this as an avenue for future work.

\begin{table}[t]
\caption{Metrics for evaluating simulations. LRSE stands for Log Relative Squared Error, and $\|\cdot\|_F$ is the Frobenius norm. $[\![\X]\!]$ represents true flag, $\D$ the true data, $[\![\hat{\X}]\!]$ the estimated flag, and $\hat{\D}$ the reconstructed data.}
    \label{tab:metrics}
    \centering
    \begin{tabular}{c|c}
       \toprule
       Metric ($\downarrow$) & Formula\\
       \midrule
       \midrule
       Chordal Distance & $\sqrt{\sum_{i=1}^k m_i - \tr(\X_i^T \hat{\X}_i\hat{\X}_i^T\X_i))}$ \\
       LRSE & $10 \log_{10} \left(\|\mathbf{D} - \hat{\mathbf{D}}\|_F^2/\|\mathbf{D}\|_F^2\right)$\\
       \bottomrule
    \end{tabular}
\end{table}
\vspace{-1mm}
\section{Results}\label{sec:results}
We run three simulations in~\cref{sec: reconstruction} to test the capacity of FD and RFD for flag recovery and reconstruction for noisy and outlier-contaminated data. In~\cref{sec:mds clustering} we visualize clusters of flag representations for hierarchically structured $\D$ matrices using FD. We highlight the utility of FD for denoising hyperspectral images in~\cref{sec: denoising_hs}. We cluster FD-extracted flag representations of hyperspectral image patches via a pixel hierarchy in~\cref{sec: clustering}. Finally, in~\cref{sec: fewshot learning}, we apply flag classifiers to three datasets for classification.

\paragraph{Baselines}
While more modern, task-specific algorithms may exist as baselines for each experiment, our primary objective is to demonstrate the effectiveness of FD and RFD (computed using~\algname) compared to the de facto standards, SVD and QR, in the context of hierarchical data. SVD is a standard denoising method. Two common flag extraction algorithms are SVD~\cite{draper2014flag,mankovich2022flag,Mankovich_2023_ICCV,szwagier2024curseisotropyprincipalcomponents} and QR~\cite{Mankovich_2023_ICCV}. In~\cref{sec: fewshot learning} we compare our results to two standard prototypes (e.g., means and subspaces) for classification within the few-shot learning paradigm.

\paragraph{Metrics}
In the additive noise model $\tilde{\D} = \D +\boldsymbol{\epsilon}$, we compute the signal-to-noise ratio (SNR) in decibels (dB) as 
\begin{equation}\label{eq:snr}
    \mathrm{SNR}(\mathbf{D},\bm{\epsilon}) = 10 \log_{10} \left(\|\mathbf{D}\|_F^2/\|\boldsymbol{\epsilon}\|_F^2\right).
\end{equation}
A negative SNR indicates more noise than signal, and a positive SNR indicates more signal than noise. The rest of our metrics are in~\cref{tab:metrics}.

\subsection{Reconstruction Under Corruption}\label{sec: reconstruction}
For both experiments, we generate $\X \in St(10,4)$ that represents $[\![ \X ]\!] \in \flag(2,4;10)$. Then we use $\X$ to build a data matrix $\D \in \R^{10 \times 40}$ with the feature hierarchy $\mathcal{A}_1 \subset \mathcal{A}_2 = \{ 1,2,\cdots,20\} \subset \{1,2,\cdots,40\}$. We generate $\tilde{\D}$ as either $\D$ with additive noise or $\D$ with columns replaced by outliers. Our goal is to recover $[\![ \X]\!]$ and $\D=\X\bR\bP^\top$ from $\tilde{\D}$ using FD and RFD with a flag type of $(2,4;10)$, and the first $4$ left singular vectors from SVD. We evaluate the estimated $[\![\hat{\X}]\!]$ and $\hat{\D}$ using~\cref{tab:metrics}.

\paragraph{Additive noise}
We contaminate with noise by $\tilde{\D} = \D  + \boldsymbol{\epsilon}$ where $\boldsymbol{\epsilon}$ is sampled from either a mean-zero normal, exponential, or uniform distribution of increasing variance. 
FD and RFD improve flag recovery over SVD and produce similar reconstruction errors (see~\cref{fig:synthetic_noise}).

\begin{figure}[t]
    \centering
    \includegraphics[width=\linewidth]{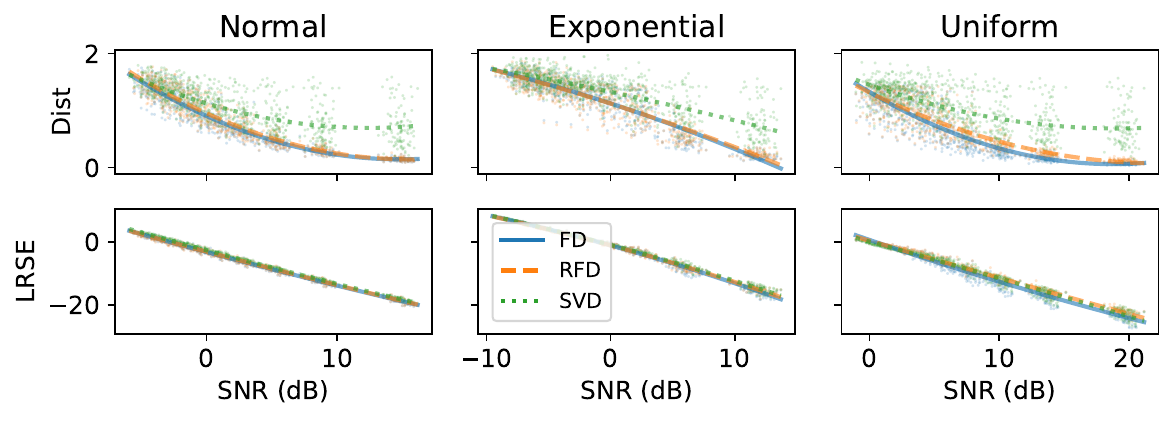}
    \caption{FD \& RFD improve flag recovery while maintaining accurate reconstructions. SNR is~\cref{eq:snr}. LSRE \& Dist are in~\cref{tab:metrics} with \emph{Dist} as the chordal distance. Best fit lines are quadratic.}
    \label{fig:synthetic_noise}
\end{figure}

\paragraph{Robustness to outliers}
We construct $\tilde{\D}$ to contain outlier columns. The inlier columns of $\tilde{\D}$ form the flag-decomposable $\D = \X \bR \bP^\top$ with the flag $[\![\X]\!]$. 
FD and RFD outperform SVD and IRLS-SVD, with RFD providing the most accurate flag recovery and inlier reconstructions (see~\cref{fig:synthetic_outliers}).

\subsection{MDS Clustering}\label{sec:mds clustering}

We generate $60$ $\D$ matrices in $3$ clusters, each with $20$ points. Then we add normally-distributed noise to generate $60$ $\tilde{\D}$ matrices (see suppl. material). We compute the SNR for each $\tilde{\D}$ via~\cref{eq:snr} and find the mean SNR for the $60$ matrices is $-4.79$ dB, indicating that significant noise has been added to each $\D$. This experiment aims to find the method that best clusters the $\tilde{\D}$ matrices.

We use SVD (with $4$ left singular vectors) and FD (with flag type $(2,4;10)$) on each of the $60$ $\tilde{\D}$ matrices to recover the flag representations. Then the chordal distance is used to generate distance matrices. Finally, MDS visualizes these data in $2$ dimensions. Our additional baseline is run on the Euclidean distance matrix between the flattened $\tilde{\D}$ matrices. We find that FD produces a distance matrix and MDS with most defined clusters in~\cref{fig:synthetic_clustering}.

\begin{figure}[t!]
    \centering
    \includegraphics[width=\linewidth]{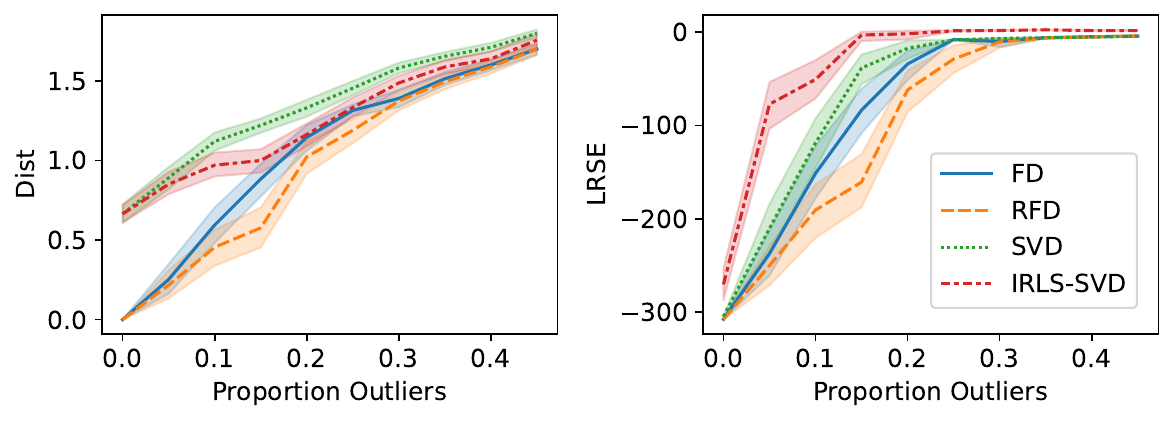}
    \caption{FD and RFD improve flag recovery and reconstruction error over SVD and IRLS-SVD with RFD. LSRE \& Dist are defined in~\cref{tab:metrics} with \emph{Dist} as the chordal distance.}
    \label{fig:synthetic_outliers}
\end{figure}

\begin{figure}[b]
\vspace{-3mm}
    \centering
    \includegraphics[width=\linewidth]{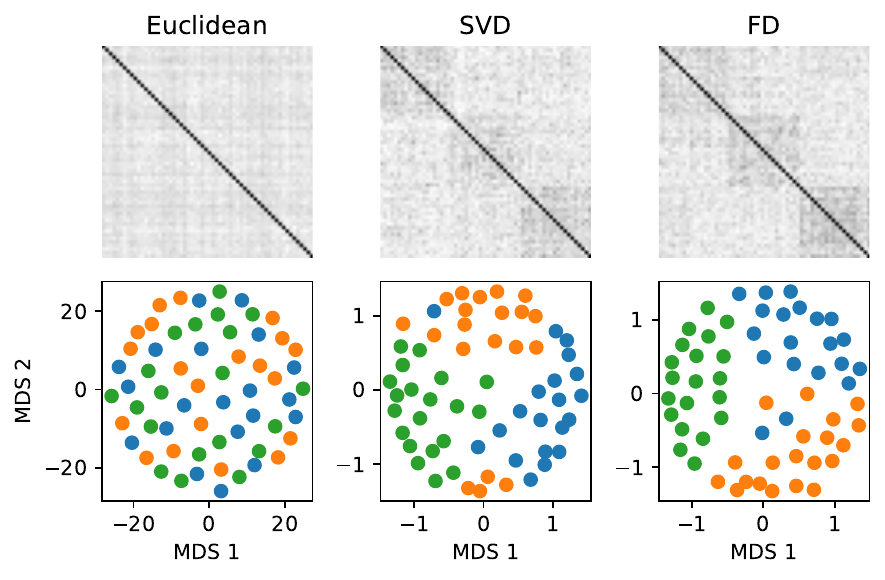}
    \vspace{-4mm}
    \caption{\textbf{(Top row)} distance matrices using Euclidean distance (Euclidean) and chordal distance between flags (SVD and FD). \textbf{(Bottom row)} 2D representation of the data colored by cluster via MDS applied to the distance matrix.}
    \label{fig:synthetic_clustering}
\end{figure}

\subsection{Hyperspectral Image Denoising}\label{sec: denoising_hs}
We consider denoising images captured by the \href{https://aviris.jpl.nasa.gov/}{AVIRIS} hyperspectral sensor. Two hyperspectral images are used for model evaluation: the KSC and Indian Pines datasets~\cite{Bandos09}. KSC is a $(512\times614)$ image with $176$ bands and Indian Pines is a $(145\times145)$ image with $200$ bands. 
We run two experiments, one on each image, by randomly selecting a $50 \times 50$ square and flattening it to generate $\D \in \R^{2500 \times p}$ (pixels as rows and bands as columns). Then, we add mean-zero Gaussian noise of increasing variance to obtain $\tilde{\D}$, on which we run our FD and SVD to denoise. LRSE is measured between $\D$ and the denoised reconstruction $\hat{\D}$ (see~\cref{tab:metrics}) to determine the quality of the reconstruction.

For our FD, we specify a flag of type $(8,9,10;2500)$, and SVD is run using the first $10$ left singular vectors. The hierarchy used as input to our algorithm mirrors the spectrum hierarchy (see~\cref{ex:spectrum_hierarchy}) by assigning $\mathcal{A}_1$ to the first $40$ bands, $\mathcal{A}_2$ to the first $100$ bands, and $\mathcal{A}_3$ to all the bands. We find in~\cref{fig:hsi denoising} that FD consistently improves HSI denoising over SVD. When testing exponential and uniform noise, FD and SVD produce similar quality denoising. 

\begin{figure}[t]
    \vspace{-.3cm}
    \centering
    \includegraphics[width=\linewidth]{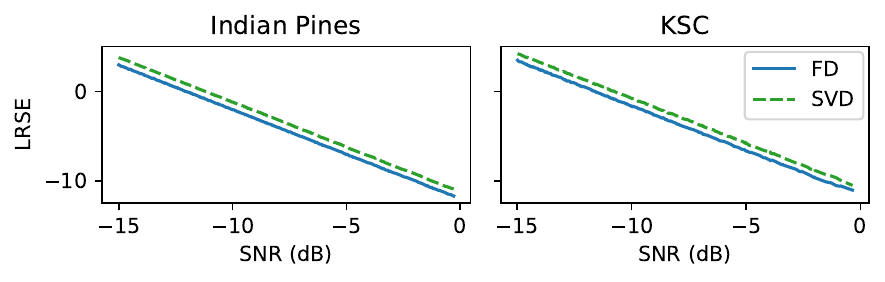}
    \caption{FD improves hyperspectral image denoising over SVD (see SNR~\cref{eq:snr}, LRSE~\cref{tab:metrics}). \vspace{-5mm}}
    \label{fig:hsi denoising}
\end{figure}

\subsection{Hyperspectral Image Clustering}\label{sec: clustering}
We now showcase image patch clustering using the KSC hyperspectral image.  
The data was pre-processed to remove low SNR and water absorption bands, then split into $3 \times 3$ patches of pixels from the same class. Each patch is translated into a $\mathbf{D} \in \R^{176 \times 9}$ (bands as rows and pixels as columns) with the hierarchy described in~\cref{ex:pixel_hierarchy} with $\cA_1$ as the center pixel. A flag recovery method is run on each $\D$ to extract a flag of type $(1,8;176)$. Then, we compute a chordal distance matrix between the collection of flags. Finally, we classify these flags using $k$-nearest neighbors. We compare FD to QR and SVD in~\cref{fig:ksc_knn} and find that FD produces the highest classification accuracy for the number of nearest neighbors between $6$ and $24$.

Instead of using chordal distance between flags to measure distance, we use a sum of Grassmannian chordal distances.
We hypothesize that this is a more suited distance for this example because it is more robust to outlier pixels. Given $[\![\X]\!],[\![\Y]\!] \in \flag(1,8;176)$, we use a chordal distance on the product of Grassmannians that takes advantage of the embedding of $\flag(1,8;176)$ in $Gr(1,176) \times Gr(7,176)$. See our suppl. material for details.

\subsection{Few-shot Learning}\label{sec: fewshot learning}
We deploy FD in few-shot learning using an Alexnet~\cite{krizhevsky2012imagenet}, pre-trained on ImageNet, as the feature extractor $f_{\Theta}: \mathcal{X} \rightarrow \R^{4096}$, admitting the representation $f_{\Theta} = f_{\Theta}^{(1)} \circ f_{\Theta}^{(2)}$ where the range of both $f_{\Theta}^{(1)}$ and $f_{\Theta}^{(2)}$ is $\R^{4096}$. We use the feature hierarchy outlined in~\cref{ex:feature_hierarchy} and the procedure in~\cref{sec:fewshot_flags} to map the support of one class to a flag prototype using FD (see~\cref{fig:fewshot_flags}). The distance between a query point and a flag prototype is~\cref{eq: fewshot_dist}. We call this pipeline a flag classifier. Our baselines include Euclidean~\cite{snell2017prototypical} and subspace~\cite{simon2020adaptive} classifiers. No fine-tuning is used to optimize the feature extractor in any experiments.

Our two baseline methods, Euclidean and subspace, use means and subspaces as prototypes. Specifically, prototypical networks~\cite{snell2017prototypical} are a classical few-shot architecture that uses averages for prototypes and Euclidean distance between prototypes and queries. On the other hand, subspace classifiers from adaptive subspace networks~\cite{simon2020adaptive} use subspaces as prototypes and measure distances between prototypes and queries via projections of the queries onto the prototype. Building upon these baseline methods, we use a flag-based approach (see~\cref{fig:concept,fig:fewshot_flags}). For a fair comparison, baselines stack features extracted by $f^{(1)}_\Theta$ and $f_\Theta$. 

We evaluate flag classifiers on EuroSat~\cite{helber2019eurosat}, CIFAR-10~\cite{krizhevsky2009learning}, and Flowers102~\cite{nilsback2008automated} datasets, and report the average classification accuracy in~\cref{tab:fewshot} over $20$ random trials each containing $100$ evaluation tasks with $10$ query images and $5$ ways per task. We find that flag classifiers perform similarly to subspace classifiers and improve classification accuracy in two cases. Further results are in suppl. material.
\begin{figure}[t]
    \centering
    \includegraphics[width=\linewidth]{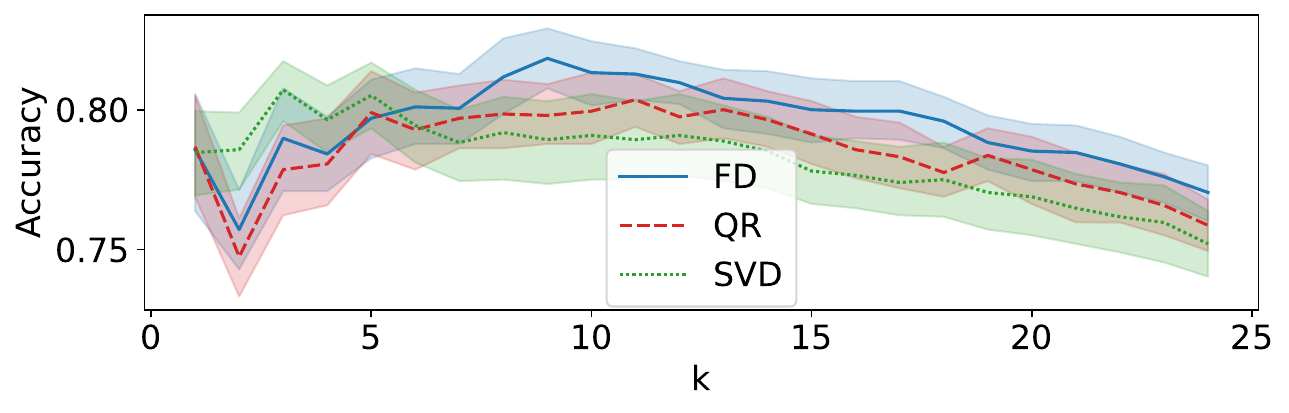}
    \caption{$k$-nearest neighbors \emph{classification accuracies} ($\uparrow$) using chordal distance matrices derived from flag representations of $3 \times 3$ image patches of the KSC dataset over $20$ random trials with a $70\%-30\%$ training-validation split. }
    \label{fig:ksc_knn}
\end{figure}
\setlength{\tabcolsep}{10pt}
\begin{table}[ht!]
    \centering    \caption{\emph{Classification accuracy ($\uparrow$)} with $s$ shots, $5$ ways, and $100$ evaluation tasks each containing $10$ query images, averaged over $20$ random trials. Flag types for `Flag' are $(s-1,2(s-1))$ and the subspace dimension is $s-1$.}
    \label{tab:fewshot}
    \setlength{\tabcolsep}{14pt}
    \resizebox{\columnwidth}{!}{
    \begin{tabular}{llccc}
    \toprule
    $s$ & Dataset & Flag & Euc. & Subsp. \\
    \midrule
    \multirow[t]{3}{*}{3} & EuroSat & \textbf{77.7}  & 76.7  & 77.6  \\
     & CIFAR-10 & \textbf{59.6}  & 58.6  & \textbf{59.6}  \\
     & Flowers102 & \textbf{90.2}  & 88.2  & \textbf{90.2}  \\
    \cline{1-5}
    \multirow[t]{3}{*}{5} & EuroSat & \textbf{81.8}  & 80.7  & \textbf{81.8}  \\
     & CIFAR-10 & \textbf{65.2}  & \textbf{65.2}  & \textbf{65.2}  \\
     & Flowers102 & \textbf{93.2}  & 91.4 & \textbf{93.2}  \\
    \cline{1-5}
    \multirow[t]{3}{*}{7} & EuroSat & \textbf{83.9}  & 82.6  & 83.8  \\
     & CIFAR-10 & 68.0  & \textbf{68.6}  & 68.1  \\
     & Flowers102 & \textbf{94.5} & 92.7 & \textbf{94.5}  \\
    \bottomrule
    \end{tabular}
    }\vspace{-3mm}
\end{table}

\vspace{-3mm}\section{Conclusion}\label{sec:conclusion}

We introduced Flag Decomposition (FD), a novel matrix decomposition that uses flags to preserve hierarchical structures within data. We further proposed~\algname~
to robustly find this decomposition even under noise and outlier contamination and studied its 
properties. With this algorithm, FD augments the machine learning arsenal by providing a robust tool for working with hierarchical data, applicable in tasks like denoising, clustering, and few-shot learning, as demonstrated by our evaluations.

\paragraph{Limitations \& future work} 
Our FD framework is a first step to hierarchy-aware decompositions and leaves ample room for future study. For example, \algname~is prone to the instabilities of Gram-Schmidt and requires a flag type and column hierarchy as inputs. In the future, we will improve algorithms for faster and more stable computation. Next, we plan to automate the flag type detection and explore fine-tuning a feature extractor for few-shot learning with flags. We will also investigate directly learning (latent) hierarchical structures from data.

\paragraph{Acknowledgements}
N. Mankovich thanks Homer Durand, Gherardo Varando, Claudio Verdun, and Bernardo Freitas Paulo da Costa for enlightening conversations on flag manifolds and their applications.
N. Mankovich and G. Camps-Valls acknowledge support from the project "Artificial Intelligence for complex systems: Brain, Earth, Climate, Society" funded by the Department of Innovation, Universities, Science, and Digital Society, code: CIPROM/2021/56. This work was also supported by the ELIAS project (HORIZON-CL4-2022-HUMAN-02-02, Grant No. 101120237), the THINKINGEARTH project (HORIZON-EUSPA-2022-SPACE-02-55, Grant No. 101130544), and the USMILE project (ERC-SyG-2019, Grant No. 855187).
T. Birdal acknowledges support from the Engineering and Physical Sciences Research Council [grant EP/X011364/1].
T. Birdal was supported by a UKRI Future Leaders Fellowship [grant number MR/Y018818/1] as well as a Royal Society
Research Grant RG/R1/241402.
The work of I. Santamaria was partly supported under grant PID2022-137099NB-C43 (MADDIE) funded by MICIU/AEI /10.13039/501100011033 and FEDER, UE.

{\small
\bibliographystyle{ieeenat_fullname} 
\bibliography{refs} 
}

\newpage
\clearpage
\vspace{2mm}
\appendix
\section*{Appendices}


We provide 
alternative methods for flag recovery in~\cref{sec:svd_and_qr},
proofs of each proposition in~\cref{sec:proofs}, a discussion of block matrix decompositions in~\cref{sec:alg_rev}, a formal presentation of the~\algname~algorithm in~\cref{sec:algs}, and additional details for the results in~\cref{sec:extra_experiments}.

\section{SVD and QR for Flag Recovery}\label{sec:svd_and_qr}
The SVD and QR decomposition recover individual subspaces of the flag $[\Q_i]$ and, in certain hierarchies, recover the entire flag $[\![\Q]\!]$. We first discuss SVD and QR for subspace recovery, then we provide examples of using each method for flag recovery.

\subsection{Subspace recovery} 
SVD and QR decomposition can be used to solve the optimization problem
\begin{equation*}
    \Q_i =  \argmin_{\X \in St(m_i,n)} \sum_{j \in \cB_j}\| \boldsymbol{\Pi}_{\X^\perp} \boldsymbol{\Pi}_{\Q^\perp_{i-1}} \cdots \boldsymbol{\Pi}_{\Q^\perp_1} \tilde{\mathbf{d}}_j \|_2^2
\end{equation*}
with $m_i = \mathrm{rank}(\boldsymbol{\Pi}_{\X^\perp} \boldsymbol{\Pi}_{\Q^\perp_{i-1}} \cdots \boldsymbol{\Pi}_{\Q^\perp_1}\B_i)$.
The QR decomposition with pivoting outputs $\boldsymbol{\Pi}_{\Q^\perp_{i-1}} \cdots \boldsymbol{\Pi}_{\Q^\perp_1} \B_i = \Q_i' \bR_i'{\bP_i'}^\top$. We then assign the columns of $\Q_i$ to columns of $\Q_i'$ associated with non-zero rows of $\bR_i'$. Using the SVD, we have $\boldsymbol{\Pi}_{\Q^\perp_{i-1}} \cdots \boldsymbol{\Pi}_{\Q^\perp_1} \B_i = \U_i \bm{\Sigma}_i \V_i^\top$. Then we assign the columns of $\Q_i$ to the columns of $\U_i$ associated with non-zero singular values.

\subsection{Flag recovery} 
Both the SVD and QR decomposition can be used for flag recovery for certain column hierarchies and flag types. In both examples, we take $\D \in \R^{n \times p}$.
\begin{exmp}[QR decomposition]
    For a tall and skinny ($p \leq n$) $\D$ with the column hierarchy $\{1,\dots,p_1\}\subset \{1,\dots,p_2\} \subset \cdots \subset \{1,\dots,p\}$ and the flag type is $(p_1,p_2,\dots,p;n)$, the QR decomposition $\D = \Q \bR$ outputs the hierarchy-preserving flag $[\![\Q]\!] \in \flag(p_1,p_2,\dots,p;n)$ because $[\mathbf{q}_1 | \mathbf{q}_2|\cdots |\mathbf{q}_i] = [\mathbf{d}_1 | \mathbf{d}_2|\cdots |\mathbf{d}_i]=[\D_{\cA_i}]$ for $i=1,2\dots,k$.
\end{exmp}

\begin{exmp}[SVD]
    Suppose $\D$ has the column hierarchy $\{1,\dots,p\}$ and the rank $n_k$. The SVD of $\D$ is $\D =\U \bm{\Sigma} \V^\top$. Let $\Q$ be the $n_k$ left singular vectors (columns of $\U$) associated with non-zero singular values. Then $[\![\Q]\!] \in \flag(n_k;n)$ is a hierarchy-preserving flag because $[\Q] = [\D]$.
\end{exmp}


\section{Theoretical Justifications}\label{sec:proofs}
We prove each proposition from the Methods section. For the sake of flow, we re-state the propositions from the Methods section before providing the proofs. Throughout these justifications we use $\mathrm{rank}(\D)$ as the dimension of the column space of $\D$. This is equivalent to the dimension of the subspace spanned by the columns of $\D$, denoted $\mathrm{dim}([\D])$

\begin{prop}\label{prop:stiefel_coords_app}
    Suppose $\cA_1 \subset \cA_2 \subset \cdots \subset \cA_k$ is a column hierarchy for $\D$. Then there exists $\Q= [\Q_1\,|\,\Q_2\,|\, \cdots\,|\,\Q_k]$ that are coordinates for the flag $[\![\Q]\!]\in\flag(n_1,n_2,\dots,n_k;n)$ where $n_i = \mathrm{rank}(\D_{\mathcal{A}_i})$ that satisfies $[\Q_i] = [\boldsymbol{\Pi}_{\Q_{i-1}^\perp}\cdots \boldsymbol{\Pi}_{\Q_1^\perp}\B_i]$ and the \textbf{projection property} (for $i=1,2\dots,k$): 
    \begin{equation}\label{eq:projection_property_app}
    \boldsymbol{\Pi}_{\Q_i^\perp}\boldsymbol{\Pi}_{\Q_{i-1}^\perp}\cdots \boldsymbol{\Pi}_{\Q_1^\perp} \B_i = \bm{0}.
    \end{equation}
\end{prop}
\begin{proof}
    For $i=1$ define $m_1 = n_1 = \mathrm{rank}(\B_1) = \mathrm{rank}(\D_{\cA_1})$. Now define $\C_1 = \B_1$ and $\Q_1 \in St(m_1,n)$ whose columns are an orthonormal basis for the column space of $\C_1$, specifically $[\Q_1] = [\C_1]$.
    
    \textbf{For ease of notation, denote} $\Q_{:i} = [\Q_1|\Q_2|\cdots|\Q_{i}]$. Define (for $i=2,3,\dots,k$) the projector onto the null space of $[\Q_1,\Q_2,\dots,\Q_i]$, as 
    \begin{equation}
        \boldsymbol{\Pi}_{\Q_{:i}^\perp} = \I - \Q_{:i}\Q_{:i}^\top.
    \end{equation}
    We use this to define $\C_i$ through 
    \begin{equation}
        \C_i = \boldsymbol{\Pi}_{\Q_{:i-1}^\perp}\B_i
    \end{equation} 
    and $\Q_i \in St(m_i,n)$ so that $[\Q_i] = [\C_i]$.

    We use mathematical induction to prove the following:
        \begin{enumerate}
        \item \emph{Non-zero} $\C_i \neq \bm{0}$,
        \item \emph{Coordinates} $\Q_{:i} = [\Q_1|\Q_2| \cdots | \Q_i]$ is in Stiefel coordinates (\eg, $\Q_{:i}\Q_{:i}^\top = \I$),
        \item \emph{Hierarchy} $[\B_1,\B_2,\dots,\B_i] = [\Q_1,\Q_2, \dots, \Q_i]$,
        \item \emph{Projection property} $\boldsymbol{\Pi}_{\Q_i^\perp} \cdots\boldsymbol{\Pi}_{\Q_1^\perp} \B_i = \bm{0}$ and \\
        $\boldsymbol{\Pi}_{\Q_i^\perp} \cdots\boldsymbol{\Pi}_{\Q_1^\perp}= \boldsymbol{\Pi}_{\Q_{:i}^\perp}$, 
        \item \emph{Dimensions} $\Q_i \in St(m_i,n)$ with $m_i = n_i - n_{i-1}$ where $n_i = \mathrm{rank}(\D_{\cA_i})$.
    \end{enumerate}

    We proceed with the base case $i=1$. (1) $\C_1 = \B_1 = \D_{\cA_1} \neq \bm{0}$. (2) $\Q_1 \in St(n_1,n)$ because its columns form an orthonormal basis for $[\C_1]$. (3) $[\B_1] = [\C_1] = [\Q_1]$. (4) Since $\boldsymbol{\Pi}_{\Q_1^\perp}$ projects into the nullspace of $\Q_1$ and $[\Q_1] = [\B_1]$, we have $\boldsymbol{\Pi}_{\Q_1^\perp} \B_1 = \bm{0}$. (5) Since $m_1 = n_1$, $n_1 = \mathrm{\dim}([\Q_1]) = \mathrm{dim}([\B_1]) = \mathrm{dim}([\D_{\cA_1}])$, and the columns of $\Q_1$ form an orthonormal basis, we have $\Q_1 \in St(m_1,n)$.
    
    Fix some $j \in \{2,3,\dots,k\}$. Suppose statements (1-5) hold true for all $i < j$.

    \paragraph{1. Non-zero} By way of contradiction, assume $\C_j = \bm{0}$. Then $\boldsymbol{\Pi}_{\Q_{:j-1}^\perp}\B_j=\bm{0}$. This means each column of $\B_j$ is in the column space of $\Q_{:j-1}$. In terms of subspaces, this implies
    \begin{equation}\label{eq:bj_subset_q}
        [\B_j] \subseteq [\Q_1,\Q_2,\dots,\Q_{j-1}] = [\B_1,\B_2,\dots,\B_{j-1}]
    \end{equation}
    where the second equality follows from the induction hypothesis part 3. \cref{eq:bj_subset_q} implies 
    \begin{equation}\label{eq:b_eq}
        \mathrm{dim}([\B_1,\B_2,\dots,\B_j])
       = \mathrm{dim}([\B_1,\B_2,\dots,\B_{j-1}]).
    \end{equation}
    By construction (see first paragraph of Methods section), 
    \begin{equation}\label{eq:span_constr}
        \mathrm{dim}([\D_{\cA_j}]) = \mathrm{dim}([\B_1,\B_2,\dots,\B_j]).
    \end{equation}
    So~\cref{eq:b_eq,eq:span_constr} imply $\mathrm{rank}(\D_{\cA_j}) = \mathrm{rank}(\D_{\cA_{j-1}})$. This contradicts the assumption of a column hierarchy for $\D$, namely $\mathrm{rank}(\D_{\cA_j}) > \mathrm{rank}(\D_{\cA_{j-1}})$.

    \paragraph{2. Coordinates} It suffices to show 
    \begin{equation}
        \Q_j^\top\Q_{:j-1} = [\Q_j^\top \Q_1| \Q_j^\top \Q_2| \cdots | \Q_j^\top \Q_{j-1}] = \bm{0}
    \end{equation} 
    which is equivalent to showing $[\Q_j]$ is orthogonal to $[\Q_1,\Q_2,\dots,\Q_{j-1}]$. By construction, 
    \begin{equation}
        [\Q_j] = [\C_j] = [\boldsymbol{\Pi}_{\Q_{:j-1}^\perp}\B_j]
    \end{equation} 
    which is orthogonal to $[\Q_1,\dots,\Q_{j-1}]$.

    \paragraph{3. Hierarchy} Using $\Q_j^\top\Q_{:j-1} = \bm{0}$, we have
    \begin{align}\label{eq:proj_prod_prop}
        \begin{aligned}
            \boldsymbol{\Pi}_{\Q_{:j}^\perp} &= \I -\Q_{:j}\Q_{:j}^\top \\
               &= \I -\sum_{\ell = 1}^j\Q_{\ell}\Q_{\ell}^\top\\
               &= \I - \Q_j \Q_j^\top - \Q_{:j-1}\Q_{:j-1}^\top \\
               &= \I - \Q_j \Q_j^\top - \Q_{:j-1}\Q_{:j-1}^\top \\
               &+ \Q_j \underbrace{\Q_j^\top \Q_{:j-1}}_{\bm{0}}\Q_{:j-1}^\top,\\
               &= (\I - \Q_j \Q_j^\top)(\I - \Q_{:j-1}\Q_{:j-1}^\top),\\
               &= \boldsymbol{\Pi}_{\Q^\perp_j} \boldsymbol{\Pi}_{\Q_{:j-1}^\perp}.\\
        \end{aligned}
    \end{align}
    By ~\cref{eq:proj_prod_prop} and the construction $[\Q_j] = [\boldsymbol{\Pi}_{\Q_{:j-1}^\perp}\B_j]$, we have
    \begin{align*}
        \bm{0} &= \boldsymbol{\Pi}_{\Q_j^\perp}\boldsymbol{\Pi}_{\Q_{:j-1}^\perp}\B_j\\
               &= \boldsymbol{\Pi}_{\Q_{:j}^\perp}\B_j,\\
               &= (\I - \Q_{:j}\Q_{:j}^\top)\B_j,\\
        \B_j   &= \Q_{:j}\Q_{:j}^\top \B_j.
    \end{align*}
    Thus $[\B_j] \subseteq [\Q_1,\Q_2,\cdots ,\Q_j]$. By the induction hypothesis (3), $[\B_1,\dots,\B_{j-1}] = [\Q_1,\dots,\Q_{j-1}]$. So $[\B_1,\B_2,\dots,\B_j] \subseteq [\Q_1,\Q_2, \dots, \Q_j]$.
    
    In contrast $[\B_j] \supseteq [\boldsymbol{\Pi}_{\Q_{:j-1}^\perp}\B_j] = [\C_j] = [\Q_j]$. So, also using $[\B_1,\dots,\B_{j-1}] = [\Q_1,\dots,\Q_{j-1}]$ (induction hypothesis 3),  we have $[\B_1,\B_2,\dots,\B_j] \supseteq [\Q_1,\Q_2,\dots,\Q_j]$. 

    \paragraph{4. Projection property} Using~\cref{eq:proj_prod_prop} and the induction hypothesis (4) that $\boldsymbol{\Pi}_{\Q_{:j-1}^\perp} = \boldsymbol{\Pi}_{\Q^\perp_{j-1}} \cdots \boldsymbol{\Pi}_{\Q^\perp_1}$, we have
    \begin{align}\label{eq:proj_identity_proof}
    \begin{aligned}
         \boldsymbol{\Pi}_{\Q_{:j}^\perp} &= \boldsymbol{\Pi}_{\Q^\perp_j} \boldsymbol{\Pi}_{\Q_{:j-1}^\perp},\\
               &= \boldsymbol{\Pi}_{\Q^\perp_j} \boldsymbol{\Pi}_{\Q^\perp_{j-1}}\cdots \boldsymbol{\Pi}_{\Q^\perp_1}.
   \end{aligned}
    \end{align}
    By construction $[\Q_j] = [\boldsymbol{\Pi}_{\Q_{:j-1}^\perp}\B_j]$. Thus 
    \begin{equation*}
        \boldsymbol{\Pi}_{\Q^\perp_j}\boldsymbol{\Pi}_{\Q_{:j-1}^\perp}\B_j = \bm{0}.
    \end{equation*}
    Using~\cref{eq:proj_identity_proof}, we have
    \begin{equation*}
        \boldsymbol{\Pi}_{\Q^\perp_j} \boldsymbol{\Pi}_{\Q^\perp_{j-1}}\cdots \boldsymbol{\Pi}_{\Q^\perp_1}\B_j = \bm{0}.
    \end{equation*}


    \paragraph{5. Dimensions} By the induction hypothesis (5), $\Q_i \in St(m_i,n)$ for $i=1,2,\dots,j-1$. So $\Q_{:j-1} \in St(n_{j-1},n)$ with $n_{j-1} = \sum_{i=1}^{j-1} m_i$. Let 
    \begin{align*}
        n_j &= \mathrm{rank}(\D_{\cA_j}), \\
        &= \mathrm{dim}([\B_1,\B_2,\dots,\B_j]),\\
        &= \mathrm{dim}([\Q_1,\Q_2,\dots,\Q_j]).
    \end{align*}
    Thus $\Q_{:j} \in \R^{n \times n_j}$ and $\Q_j \in \R^{n \times m_j}$ with $m_j = n_j - n_{j-1}$. $\Q_j$ has orthonormal columns by construction, so $\Q_j \in St(m_j,n)$.

    By way of mathematical induction, we have proven (1-5) for all $i=1,2,\dots,k$. Specifically, given a column hierarchy on $\D$, we have found coordinates for a hierarchy-preserving flag $[\![\Q]\!] \in \flag(n_1,n_2,\dots,n_k;n)$ that satisfies the projection property.
\end{proof}

Although we can write $\bR_{i,j} = \Q_i^\top\B_j$ for $j \geq i$, an equivalent definition is provided in~\cref{eq:R_and_P_app} because it is used in~\cref{alg:FD}.

\begin{prop}\label{prop:proj_prop_and_flags_app}
 Suppose $\cA_1 \subset \cA_2 \subset \cdots \subset \cA_k$ is a column hierarchy for $\D$. Then there exists some hierarchy-preserving 
 $[\![\Q]\!] \in \flag(n_1,n_2,\dots, n_k;n)$ (with $n_i = \mathrm{rank}(\D_{\mathcal{A}_i})$) 
 that satisfies the projection property of $\D$ and can be used for a flag decomposition of $\D$ with
    \begin{align}\label{eq:R_and_P_app}
    \begin{aligned}
    \bR_{i,j} &= 
        \begin{cases}
            \Q_i^\top\boldsymbol{\Pi}_{\Q_{i-1}^\perp}\cdots \boldsymbol{\Pi}_{\Q_1^\perp} \B_i, &i=j\\
            \Q_i^\top\boldsymbol{\Pi}_{\Q_{i-1}^\perp}\cdots \boldsymbol{\Pi}_{\Q_1^\perp} \B_j, &i < j
        \end{cases},\\
        \bP_i &= \left[ \,\mathbf{e}_{b_{i,1}}\,|\, \mathbf{e}_{b_{i,2}}\,|\, \cdots\,|\, \mathbf{e}_{b_{i,|\cB_i|}} \right]
    \end{aligned}
    \end{align}
    where $\{b_{i,j}\}_{j=1}^{|\cB_i|} = \cB_i$ and $\mathbf{e}_{b}$ is the $b_{i,j}$$^{\mathrm{th}}$ standard basis vector.
\end{prop}
\begin{proof}
    We define the permutation matrix $\bP = [\bP_1|\bP_2|\cdots |\bP_k]$ in~\cref{eq:R_and_P_app}. Specifically, we assign the non-zero values in each column of $\bP_i$ to be the index of each element in $\mathcal{B}_i$. In summary, $\bP_i$ is defined so that $\D\bP= [\B_1|\B_2|\cdots|\B_k]$ and $\D = \B = [\B_1|\B_2|\cdots|\B_k]\bP^\top$.

    We find the coordinates $[\Q_1|\Q_2|\cdots|\Q_k] \in St(n_k,n)$ for the hierarchy-preserving flag $[\![\Q]\!] \in \flag(n_1,n_2,\dots,n_k;n)$ with $n_k = \mathrm{rank}(\D_{\cA_i})$ that satisfies the projection property using~\cref{prop:stiefel_coords_app}. 
    
    Now, we aim to find $\bR$ so that $\B = \Q \bR$. Using the projection property $\boldsymbol{\Pi}_{\Q_j^\perp}\cdots\boldsymbol{\Pi}_{\Q_1^\perp}\B_j = \bm{0}$ and the identity $\boldsymbol{\Pi}_{\Q_j^\perp}\cdots\boldsymbol{\Pi}_{\Q_1^\perp} = \boldsymbol{\Pi}_{[\Q_1|\cdots|\Q_j]^\perp}$ from~\cref{eq:proj_identity_proof}, we can write 
    \begin{equation}\label{eq:easy_r}
        \B_j = \sum_{i=1}^j\Q_i \Q_i^\top \B_j =  \sum_{i=1}^j\Q_i\bR_{i,j}.
    \end{equation}
    This is equivalent to the projection formulation of $\bR_{i,j}$ in~\cref{eq:R_and_P_app} because (for $i=1,2,\dots,k$),
    \begin{align}\label{eq:projection_qit}
    \begin{aligned}
        &\Q_i^\top\boldsymbol{\Pi}_{\Q_{i-1}^\perp}\cdots \boldsymbol{\Pi}_{\Q_1^\perp} \\
        &= \Q_i^\top\boldsymbol{\Pi}_{[\Q_{i-1}| \cdots|\Q_1]^\perp},\\
        &= \Q_i^\top(\I - [\Q_{i-1}| \cdots|\Q_1][\Q_{i-1}| \cdots|\Q_1]^\top),\\
        &= \Q_i^\top - \underbrace{\Q_i^\top [\Q_{i-1}| \cdots|\Q_1]}_{\bm{0}}[\Q_{i-1}| \cdots|\Q_1]^\top,\\
        &= \Q_i^\top. 
        \end{aligned}
    \end{align}
    Stacking the results from~\cref{eq:easy_r,eq:projection_qit} into block matrices gives $\B = \Q\bR$ with $\bR$ defined in~\cref{eq:R_and_P_app}.
\end{proof}

\begin{prop}
    A data matrix $\D$ admits a flag decomposition of type $(n_1,n_2, \cdots, n_k; n)$ if and only if $\cA_1 \subset \cA_2 \subset \cdots \subset \cA_k$ is a column hierarchy for $\D$.
\end{prop}

\begin{proof}
    We first tackle the forward direction. Suppose $\D$ admits a flag decomposition with the hierarchy $\cA_1 \subset \cA_2 \subset \cdots \subset \cA_k$. Then $\D = \Q \bR \bP^\top$ and $\D \bP = \Q \bR$ because $\bP$ is a permutation matrix. 
    Define 
    \begin{equation}
        \B = [\B_1| \B_2 | \cdots | \B_k] = \D \bP = \Q\bR.
    \end{equation} 
    Since we have a flag decomposition, $[\![\Q]\!] \in \flag(n_1,n_2,\dots,n_k;n)$ with $\Q = [\Q_1|\Q_2|\cdots|\Q_k] \in St(n_k,n)$. Since $\Q$ is in Stiefel coordinates we have $\Q_{j}^T\Q_{i} = \bm{0}$ for all $j < i$, so
    \begin{equation}\label{eq:q_ineq}
        \mathrm{dim}([\Q_1, \Q_2, \dots,\Q_{i-1}]) < \mathrm{dim}([\Q_1, \Q_2, \dots,\Q_i]).
    \end{equation} 
    Since $[\![\Q]\!]$ is hierarchy preserving, for $i=1,2,\dots,k$ we have $[\B_1,\B_2,\dots,\B_i] = [\Q_1, \Q_2, \dots,\Q_i]$. Using this and ~\cref{eq:q_ineq}, we have
    \begin{equation}\label{eq:dim_ineq_b}
        \mathrm{dim}([\B_1, \B_2, \dots,\B_{i-1}]) < \mathrm{dim}([\B_1, \B_2, \dots,\B_i]).
    \end{equation}
    By construction $\mathrm{dim}([\B_1,\B_2,\dots,\B_i]) = \mathrm{dim}([\D_{\cA_i}])$. So, using~\cref{eq:dim_ineq_b}, we have shown $\mathrm{dim}([\D_{\cA_{i-1}}]) < \mathrm{dim}([\D_{\cA_i}])$ proving $\cA_1 \subset \cA_2 \subset \cdots \subset \cA_k$ is a column hierarchy for $\D$.

    The backward direction is proved in~\cref{prop:stiefel_coords_app,prop:proj_prop_and_flags_app}. Specifically, given a data matrix with an associated column hierarchy,~\cref{prop:stiefel_coords_app} describes how to find a hierarchy-preserving flag. Then~\cref{prop:proj_prop_and_flags_app} shows how to find the permutation matrix $\bP$ from the column hierarchy and the weight matrix $\bR$ from $\Q$ so that $\D = \Q \bR \bP^\top$.
    \end{proof}


Recall the two optimization problems proposed in the Methods section:
\begin{equation}\label{eq:general_opt_app}
    [\![\Q]\!] = \argmin_{[\![\X]\!] \in \flag(n_1,n_2, \dots, n_k;n)} \sum_{i=1}^k \sum_{j \in \cB_i}\| \boldsymbol{\Pi}_{\X_i^\perp} \cdots \boldsymbol{\Pi}_{\X_1^\perp} \tilde{\mathbf{d}}_j \|_r^q,
\end{equation}
\begin{equation}\label{eq:iterative_opt_app}
    \Q_i =  \argmin_{\X \in St(m_i,n)} \sum_{j \in \cB_j}\| \boldsymbol{\Pi}_{\X^\perp} \boldsymbol{\Pi}_{\Q^\perp_{i-1}} \cdots \boldsymbol{\Pi}_{\Q^\perp_1} \tilde{\mathbf{d}}_j \|_r^q.
\end{equation}

\begin{prop}[Block rotational ambiguity]
    Given the FD $\D = \Q \bR \bP^\top$, any other Stiefel coordinates for the flag $[\![\Q]\!]$ produce an FD of $\D$ (via~\cref{prop:proj_prop_and_flags_app}). Furthermore, different Stiefel coordinates for $[\![\Q]\!]$ produce the same objective function values in~\cref{eq:general_opt_app,eq:iterative_opt_app} (for $i=1,\cdots,k$).
\end{prop}
\begin{proof}
    The flag manifold $\flag(n_1,n_2,\dots,n_k;n)$ is diffeomorphic to $St(n_k,n)/(O(m_1)\times \cdots \times O(m_k))$ where $m_i = n_i - n_{i-1}$. Suppose $\D = \Q \bR \bP^\top$ is a flag decomposition. Consider $\Q \mathbf{M} \in St(n_k,n)$ with $\mathbf{M} = \mathrm{diag}([\mathbf{M}_1|\mathbf{M}_2|\cdots | \mathbf{M}_k]) \in O(m_1)\times \cdots \times O(m_k)$, meaning $\M_i \in O(m_i)$ for $i=1,2,\dots,k$. 
    
    Notice $\Q$ and $\Q \mathbf{M}$ are coordinates for the same flag, $[\![\Q]\!]=[\![\Q\M]\!]$. 
    
    The key property for this proof is that right multiplication by $\M_i$ does not change projection matrices $\Q_i\Q_i^\top$. Specifically $\Q_i \Q_i^\top = \Q_i\M_i (\Q_i\M_i)^\top$ for $i=1,2,\dots,k$. 
    
    Both $\Q$ and $\M\Q$ satisfy the projection property relative to $\D$ because (for $i=1,2,\dots,k$)
    \begin{equation}\label{eq:proj_mat_equiv}
    \boldsymbol{\Pi}_{\Q_i^\perp} = \I - \Q_i \Q_i^\top = \I - \Q_i \mathbf{M}_i (\mathbf{M}_i \Q_i)^\top = \boldsymbol{\Pi}_{(\Q_i\mathbf{M}_i)^\perp}
    \end{equation}
    which implies that the objective function values in~\cref{eq:general_opt_app,eq:iterative_opt_app} (for $i=1,2,\dots,k$) are the same for $\Q$ and $\Q\mathbf{M}$. Additionally,~\cref{eq:proj_mat_equiv} implies
    \begin{align}
        \bm{0} &=\boldsymbol{\Pi}_{\Q_i^\perp}\boldsymbol{\Pi}_{\Q_{i-1}^\perp}\cdots \boldsymbol{\Pi}_{\Q_1^\perp} \B_i \\
        &= \boldsymbol{\Pi}_{(\Q_i\mathbf{M}_i)^\perp}\boldsymbol{\Pi}_{(\Q_{i-1}\mathbf{M}_{i-1})^\perp}\cdots \boldsymbol{\Pi}_{(\Q_1\mathbf{M}_1)^\perp} \B_i.
    \end{align}
    
    Since $[\![\Q]\!]$ is hierarchy-preserving and rotations do not change subspaces, we have $[\Q_1,\Q_2, \dots, \Q_i] = [\Q_1\M_1,\Q_2\M_2,\dots,\Q_i \M_i]$. Thus $[\![\Q \M]\!]$ is hierarchy-preserving.
    
    Define $\bR^{(\mathbf{M})}$ with blocks $\bR^{(\mathbf{M})}_{i,j} = (\Q_i\mathbf{M}_i)^\top\B_j$. Notice 
    \begin{align}
        \B_j &= \sum_{i=1}^j \Q_i \bR_{i,j},\\
             &= \sum_{i=1}^j \Q_i \Q_i^\top  \B_j,\\
             &= \sum_{i=1}^j (\Q_i \M_i) (\Q_i\M_i)^\top \B_j,\\
             &= \sum_{i=1}^j (\Q_i \M_i)\bR^{(\mathbf{M})}_{i,j}.
    \end{align}
    Thus $\D = (\Q\mathbf{M})\bR^{(\mathbf{M})} \bP^\top$ is a hierarchy-preserving flag decomposition.
\end{proof}

\section{Relationship to MLMD~\cite{ong2016beyond}}\label{sec:alg_rev}
The Multiscale Low Rank Matrix Decomposition (MLMD)~\cite{ong2016beyond} models $\D = \sum_i \X_i$ where each block low-rank matrix $\X_i$ models finer-grained features than $\X_{i+1}$. Suppose $\D = [\B_1|\B_2]\in \R^{n \times p}$ is of rank $n_k$ with columns sorted according to the hierarchy $\cA_1 \subset \cA_2$. The FD with flag type $(n_1,n_2;n)$ is $\D = \Q \bR$ where $\Q = [\Q_1 | \Q_2] \in St(n_k,n)$, $\Q_1 \in \R^{n \times n_1}$, and $\bR$ is block upper triangular. FD does not seek block low-rank representations for different scales, rather it extracts a hierarchy-preserving flag $[\![\Q]\!] \in \flag(n_1,n_2;n)$. Moreover, MLMD partitions $\D$ into column blocks requiring the block partition $P_2$ to be an `order of magnitude' larger than $P_1$ ($1$st par. Sec. II). 
FD is more general and free of this restriction. 
MLMD models $\D = \X_1 + \X_2$ with $\X_i = \sum_{b \in P_i} R_b (\U_b \boldsymbol{\Sigma}_b \V^\top_b)$ where $R_b$ is a block reshaper. The output would be $3$ bases (in each $\U_b$), two for the columns of $\B_1$ and $\B_2$, and one for all of $\D$. These are neither mutually orthogonal nor guaranteed to be hierarchy-preserving. FD outputs one basis in the columns of $\Q=[\Q_1|\Q_2]$ are hierarchy-preserving: $[\Q_1]=[\B_1]$, and $[\Q] = [\D]$.

\section{Algorithms}\label{sec:algs}
Our get\_basis algorithm extracts $\Q_i \in St(m_i,n)$ from $\mathbf{C}_i \in \R^{n \times |\cB_i|}$ so that $[\Q_i] = [\C_i]$ by solving the optimization inspired by~\cref{eq:iterative_opt_app}:
\begin{equation}
    \Q_i = \argmin_{\X \in St(m_i,n)} \sum_{j=1}^{|\cB_i|}\| \boldsymbol{\Pi}_{\X^\perp} \mathbf{c}_j^{(i)}\|_2^q
\end{equation}
for $q=1,2$. We use $\mathbf{c}_j^{(i)}$ to denote the $j$th column of $\C_i$ and $\cB_i = \cA_i \setminus \cA_{i-1}$. A naive implementation of IRLS-SVD addresses $q=1$ and SVD addresses $q=2$.

\begin{algorithm}[ht!]
\caption{get\_basis}\label{alg:get_basis}
 \textbf{Input}: {$\mathbf{C}_i \in \R^{n \times |\cB_i|}$, $m_i \in \R$ (optional)}\\
 \textbf{Output}: {$\X_i \in \R^{m_i}$} \\[0.25em]
     \If{SVD}{
           $\U \boldsymbol{\Sigma} \V^T \gets \mathrm{SVD}(\C_i)$;\\
           \If{$m_i$ is none}
           {
           $m_i \gets \mathrm{rank}(\C_i)$;\\
           }
           $\Q_i \gets \U(1:\mathrm{end},1:m_i)$;\\
        }
     \If{IRLS-SVD}{
        \While{not converged}{
            \For{$j \gets 1$ \KwTo $|\cB_i|$}{
            $\mathbf{c}_j^{(i)} \gets  \C_i(1:\mathrm{end},j)$;\\
            $w_j \gets \mathrm{max}\left(\|\mathbf{c}_j^{(i)}-  \Q_i \Q_i^\top \mathbf{c}_j^{(i)}\|_2,10^{-8}\right)^{-1/2}$;\\
            }
            $\mathbf{W}_i \gets\mathrm{diag}(w_1,w_2,\dots,w_{|\cB_i|})$;\\
            $\U \boldsymbol{\Sigma} \V^T \gets \mathrm{SVD}(\C_i\mathbf{W}_i)$;\\
           \If{$m_i$ is none}{
           $m_i \gets \mathrm{rank}(\C_i\mathbf{W}_i)$;\\
           }
           $\Q_i \gets \U(1:\mathrm{end},1:m_i)$;\\
        }
        }
\end{algorithm}

\algname~is essentially BMGS~\cite{jalby1991stability} with a different get\_basis function. The get\_basis in~\cref{alg:get_basis} is used at each iteration of~\algname~to extract a $\Q_i \in St(m_i,n)$ so that $[\Q_i] = [\boldsymbol{\Pi}_{\Q_{i-1}^\perp} \cdots \boldsymbol{\Pi}_{\Q_1^\perp}\B_i]$. The second nested for loop in~\algname~defines $\bR_{i,j}$ using~\cref{eq:R_and_P_app} and updates $\B_j$ so that, at iteration $i$, we take run $\mathrm{get\_basis}$ on $\C_i = \boldsymbol{\Pi}_{\Q_{i-1}^\perp} \cdots \boldsymbol{\Pi}_{\Q_1^\perp}\B_i$. 

\begin{algorithm}[ht!]
\caption{\algname}\label{alg:FD}
 \textbf{Input}: {A data matrix $\mathbf{D} \in \R^{n \times p}$, \\
 c. hierarchy $\cA_1\subset \cA_2 \subset \cdots \subset \cA_k = \{1,2,\dots,p\}$, \\
 flag type $(n_1,n_2,\cdots,n_k;n)$ with $n_k \leq p$}\\
 \textbf{Output}: {Hierarchy-preserving flag $[\![\Q]\!] \in \flag(n_1,n_2,\dots,n_k;n)$, \\
 weights $\bR \in \R^{n_k \times p}$, perm. mat. $\bP \in \R^{p \times p}$} \\
 with $\D = \Q \bR \bP^\top$\\[0.25em]
     \For{$i\gets1$ \KwTo $k$}{
         $\cB_i \gets \cA_i \setminus \cA_{i-1}$; \\
         $\mathbf{B}_i  \gets \mathbf{D}(1:\mathrm{end},\cB_i) \in \R^{n \times |\cB_i|}$;\\
         $\bP_i \gets  \left[ \mathbf{e}_{b_{i,1}}| \mathbf{e}_{b_{i,2}}| \cdots| \mathbf{e}_{b_{i,|\cB_i|}} \right]$
        }
     \For{$i\gets1$ \KwTo $k$}{
         $m_i \gets n_i-n_{i-1}$;\\
         $\Q_i\gets \mathrm{get\_basis}(\B_i, m_i)$;\\
         $\bR_{i,i} \gets \Q_i^\top \B_i$;\\
         \For{$j\gets i+1$ \KwTo $k$}{
             $\bR_{i,j} \gets \Q_i^\top \B_j$; \%assign $\bR_{i,j}$\\
             $\B_j \gets \B_j - \Q_i \bR_{i,j}$; \%project: $\B_j$ into nullspace of $\Q_i$\\
         }
        }
     $\Q \gets \begin{bmatrix}\Q_1 | \Q_2 | \cdots | \Q_k\end{bmatrix}$;\\
     $\bR \gets \begin{bmatrix}
            \bR_{11} & \bR_{12} & \cdots & \bR_{1k}\\
            \mathbf{0} & \bR_{22} & \cdots & \bR_{2k}\\
            \vdots & \vdots & \ddots & \vdots\\
            \mathbf{0} & \mathbf{0} & \cdots & \bR_{kk} \end{bmatrix}$;\\
     $\bP \gets \begin{bmatrix} \bP_1 | \bP_2 | \cdots | \bP_k \end{bmatrix}$;
\end{algorithm}
\begin{remark}[~\algname~operations count]
    In this remark, we use $O$ to denote big-$O$ notation and not the orthogonal group. We also denote $b_i = |\cB_i|$ so $\B_i \in \R^{n \times b_i}$.
    
     The operations count for the SVD of a matrix $\B_i$ is $O(nb_i\mathrm{min}(n,b_i))$. FD runs $k$ SVDs for each piece of the column hierarchy. Thus its operations count is $O\left(n\sum_{i=1}^kb_i\mathrm{min}(n,b_i)\right)$.

    The IRLS-SVD operations count is $O(c_inb_i\mathrm{min}(n,b_i))$ where $c_i$ is the number of iterations until convergence for IRLS-SVD on $\B_i$. Since IRLS-SVD is run $k$ times in Robust FD, the operations count is $O\left(n\sum_{i=1}^kc_ib_i\mathrm{min}(n,b_i)\right)$.
\end{remark}

We summarize the properties for flag recovery methods in~\cref{tab:alg_table2}.
\begin{table}[ht!]
    \centering
    \caption{A summary of flag recovery methods and their properties.}
    \footnotesize
    \resizebox{\columnwidth}{!}{
    \label{tab:alg_table2}
    \begin{tabular}{c|c@{\:\:\:\:\:\:\:\:}c@{\:\:}c@{\:\:}c@{\:\:\:\:\:\:\:\:}c}
        \toprule
        Decomp. & QR & SVD & IRLS-SVD & FD  & RFD\\
        \midrule
        Robust & \xmark & \xmark & \cmark & \xmark & \cmark \\ 
        Order-pres. & \cmark & \xmark & \xmark & \cmark & \cmark \\ 
        Flag-type & \xmark & \xmark & \xmark & \cmark & \cmark \\ 
        Hier.-pres. & \xmark & \xmark & \xmark & \cmark & \cmark \\ 
        \bottomrule
    \end{tabular}
    }
\end{table}

\section{Results}\label{sec:extra_experiments}
We first describe data generation for each simulation. Then we provide details on the hyperspectral image clustering experiment and confidence intervals for few-shot learning.

\subsection{Reconstruction Simulations}
We consider either additive noise to the data or data contamination with outliers. For both experiments, we generate a Stiefel matrix $[ \X_1 | \X_2 ] = \X \in St(10,4)$ that represents $[\![ \X ]\!] \in \flag(2,4;10)$. Then we generate the data matrix $\D$ with the feature hierarchy $\mathcal{A}_1 \subset \mathcal{A}_2 = \{ 1,2, \cdots, 20\} \subset \{1,2, \cdots, 40\}$. We attempt to recover $[\![ \X]\!]$ and $\D = [\B_1| \B_2] \in \R^{10 \times 40}$ using FD and Robust FD with a flag type of $(2,4;10)$, and the first $4$ left singular vectors from SVD. We evaluate the estimated $[\![\hat{\X}]\!]$ and $\hat{\D}$ using chordal distance and LRSE.

\paragraph{Additive noise}
We consider the following model for $\D$:
\begin{equation*}
    \mathbf{d}_i = \begin{cases}
        \X_1 \mathbf{s}_{1i}, & i \in \mathcal{B}_1 = \{1,\cdots, 20\}\\
        \X\mathbf{s}_{2i}, & i \in \mathcal{B}_2 = \{21,\cdots, 40\}\\
    \end{cases}
\end{equation*}
where each entry of $\mathbf{s}_{ji}$ from a normal distribution with mean $0$ and variance $1$. We contaminate $\D$ with noise by $\tilde{\D} = \D  + \boldsymbol{\epsilon}$ where $\boldsymbol{\epsilon}$ is sampled from either a normal, exponential, or uniform distribution of increasing variance. The goal is to recover $\D$ and $\X$ from $\tilde{\D}$. FD and Robust FD improve flag recovery over SVD and produce similar reconstruction errors. 

\paragraph{Outliers columns}
We randomly sample a subset of columns of $\tilde{\D}$ to be in the set of outliers $\mathcal{O}$. Each outlier column is in the nullspace of $[\X]$ and each entry of $\mathbf{o}_i$ is sampled from a normal distribution with mean $0$ and variance $1$. We use the same scheme as the additive noise case to sample $\mathbf{s}_{ji}$. Using these quantities, we sample the $i^{\mathrm{th}}$ column of $\tilde{\D}$ as
\begin{equation*}
    \tilde{\mathbf{d}}_i = \begin{cases}
        \X_1\mathbf{s}_{1i}, & i \in \mathcal{B}_1\setminus \mathcal{O}\\
        \X\mathbf{s}_{2i}, & i \in \mathcal{B}_2\setminus \mathcal{O}\\
        (\I - \X\X^T)\mathbf{o}_i, & i \in \mathcal{O}.\\
    \end{cases}
\end{equation*}
We define $\D$ as the matrix containing only inlier columns of $\tilde{\D}$. We attempt to recover $\D$ and $[\![\hat{\X}]\!]$ from $\tilde{\D}$. We measure the chordal distance between our estimated $[\![\hat{\X}]\!]$ and $[\![\X]\!]$ and the LRSE between our inlier estimates $\hat{\D}$ and $\D$.


\subsection{Clustering Simulation}
We generate three Stiefel matrices to serve as centers of our clusters $\left[ \X_1^{(c)} | \X_2^{(c)} \right]= \X^{(c)} \in St(4,10)$ that represent $[\![ \X^{(c)} ]\!] \in \flag(2,4;10)$ for $c = 1,2,3$. We use each of these centers to generate $20$ $\D$-matrices with the feature hierarchy $\mathcal{A}_1 = \{ 1,2, \cdots, 20\}$, $\cA_2 = \{1,2, \cdots, 40\}$ in each cluster. The $i$th column in cluster $c$ of the data matrix $\D_i^{(c)}$ is generated as
\begin{equation}
    \mathbf{d}_i^{(c)} = \begin{cases}
        \X_1^{(c)}\mathbf{s}_{1i} , & i \in \mathcal{B}_1\\
        \X^{(c)}\mathbf{s}_{2i} , & i \in \mathcal{B}_2.\\
    \end{cases}
\end{equation}
Then we generate the detected data matrices as $\tilde{\D}_i^{(c)} = \D_i^{(c)}+\boldsymbol{\epsilon}_{i}^{(c)}$. We sample $\boldsymbol{\epsilon}_{1i}^{(c)}$ and $\boldsymbol{\epsilon}_{2i}^{(c)}$ from a normal distribution with mean $0$ and standard deviation $.95$ and $\mathbf{s}_{1i}$ and $\mathbf{s}_{2i}$ from a normal distribution with mean $0$ and standard deviation $1$.

\subsection{Hyperspectral image clustering}
A total of 326 patches were extracted, each with a shape of ($3 \times 3$), with the following distribution: 51 patches of class Scrub, 7 of Willow swamp, 12 of Cabbage palm hammock, 10 of Cabbage palm/oak hammock, 11 of Slash pine, 13 of Oak/broad leaf hammock, 7 of Hardwood swamp, 20 of Graminoid marsh, 39 of Spartina marsh, 25 of Cattail marsh, 29 of Salt marsh, 24 of Mudflats, and 78 of Water.

In this experiment, we measure the distance between two flags $[\![\X]\!],[\![\Y]\!]$ as
\begin{equation}
    \frac{1}{\sqrt{2}}\|\X_1\X_1^T - \Y_1 \Y_1^T\|_F + \frac{1}{\sqrt{2}}\|\X_2\X_2^T - \Y_2 \Y_2^T\|_F.
\end{equation}

\subsection{Few-shot learning}
We now expand on the methodological details of the baseline methods for few-shot learning and report further results including standard deviations.

\paragraph{Prototypical networks}
Prototypical networks~\cite{snell2017prototypical} are a classical few-shot architecture that uses averages for class representatives and Euclidean distance for distances between representatives and queries. Specifically, a prototype for class $c$ is
\begin{equation}
    \mathbf{q} = \frac{1}{s} \sum_{i=1}^s f_\Theta(\bm{x}_{c,i})
\end{equation}
and the distance between a query point, $f_{\Theta}(\bm{x})$, is
\begin{equation}
    \|\mathbf{q} - f_{\Theta}(\bm{x})\|_2^2.
\end{equation}
In experiments, we refer to this method as `Euc.'

\paragraph{Subspace classifiers}
Subspace classifiers from adaptive subspace networks~\cite{simon2020adaptive} use subspace representatives and measure distances between subspace representatives and queries via projections of the queries onto the subspace representatives. Although the original work suggests mean subtraction before computing subspace representatives and for classification, we notice that there is no mean-subtraction in the code provided on \href{https://github.com/chrysts/dsn_fewshot/blob/master/Resnet12/models/classification_heads.py}{GitHub}. Therefore, we summarize the model used on GitHub as
\begin{equation}
    \tilde{\bm{X}}_c = \left[ f_\Theta(\bm{x}_{c,1})|f_\Theta(\bm{x}_{c,2})|\cdots| f_\Theta(\bm{x}_{c,s}) ,\right]
\end{equation}
\begin{equation}
    \mathbf{U}_c \boldsymbol{\Sigma}_c \mathbf{V}_c^\top = \tilde{\bm{X}}_c,
\end{equation}
\begin{equation}
    \mathbf{Q}_c = \mathbf{U}_c(1:\mathrm{end},1:s-1).
\end{equation}
We say that the span of the columns of $\Q_c$ serves as the subspace representative for class $c$. This can be seen as a mapping of a set feature space representation of the shots from one class to $Gr(s-1,n)$ via the SVD. The distance between a query $f_{\Theta}(\bm{x})$ and class $c$ is 
\begin{equation}
    \|f_{\Theta}(\bm{x}) - \mathbf{Q}_c \mathbf{Q}_c^\top f_{\Theta}(\bm{x})\|_F^2.
\end{equation}
This is the residual of the projection of a query point onto the subspace representative for class $c$.

\paragraph{Stacking features}
Our application of flag classifiers uses an alexnet backbone $f_\Theta = f^{(2)}_\Theta \circ f^{(1)}_\Theta$. Given a sample $\bm{f}$, flag classifiers leverage both the information extracted by $f_\Theta$ \emph{and} $f^{(1)}_\Theta$. This is already an advantage over the baseline methods because flag classifiers see more features. Therefore, we modify prototypical network and subspace classifiers for a fair baseline to flag nets. Specifically, we replace $f_\Theta(\bm{x})$ with 
\begin{equation}
    \begin{bmatrix}
        f^{(1)}_\Theta(\bm{x})\\
        f_\Theta(\bm{x})
    \end{bmatrix}.
\end{equation}
This doubles the dimension of the extracted feature space and thereby exposes these algorithms to problems like the curse of dimensionality. Additionally, it assumes \emph{no order} on the features extracted by $f_\Theta$ and $f^{(1)}_\Theta$ therein not respecting the natural hierarchy of the alexnet feature extractor.

\paragraph{Further results}
We provide the classification accuracies along with standard deviations over $20$ random trials in~\cref{tab:fewshot_app1,tab:fewshot_app2}. 
\setlength{\tabcolsep}{1pt}
\begin{table}[ht]
    \centering    \caption{\emph{Classification accuracy ($\uparrow$)} with $s$ shots, $5$ ways, and $100$ evaluation tasks each containing $10$ query images, averaged over $20$ random trials. Flag types for `Flag' are $(s-1,2(s-1))$ and the subspace dimension is $s-1$. Baselines see stacked features from both $f^{(1)}_\Theta$ and $f_\Theta$.}
    \label{tab:fewshot_app1}
    \resizebox{\columnwidth}{!}{    
    \footnotesize
    \begin{tabular}{@{\hskip 4pt}l@{\hskip 4pt}l@{\hskip 4pt}c@{\hskip 7pt} c @{\hskip 7pt}c}
    \toprule
    $s$ & Dataset & Flag & Euc. & Subsp. \\
    \midrule
    \multirow{3}{*}{3} & EuroSat & $\bm{77.7} \pm 1.0$ & $76.7 \pm 1.0$ & $77.6 \pm 1.0$ \\
     & CIFAR-10 & $\bm{59.6} \pm 1.0$ & $58.6 \pm 0.9$ & $\bm{59.6} \pm 1.0$ \\
     & Flowers102 & $\bm{90.2} \pm 0.7$ & $88.2 \pm 1.0$ & $\bm{90.2} \pm 0.7$ \\
    \cline{1-5}
    \multirow{3}{*}{5} & EuroSat & $\bm{81.8} \pm 0.7$ & $80.7 \pm 0.8$ & $\bm{81.8} \pm 0.7$ \\
     & CIFAR-10 & $\bm{65.2} \pm 0.9$ & $\bm{65.2} \pm 0.9$ & $\bm{65.2} \pm 0.9$ \\
     & Flowers102 & $\bm{93.2} \pm 0.5$ & $91.4 \pm 0.6$ & $\bm{93.2} \pm 0.5$ \\
    \cline{1-5}
    \multirow{3}{*}{7} & EuroSat & $\bm{83.9} \pm 0.8$ & $82.6 \pm 0.8$ & $83.8 \pm 0.8$ \\
     & CIFAR-10 & $68.0 \pm 0.7$ & $\bm{68.6} \pm 0.8$ & $68.1 \pm 0.7$ \\
     & Flowers102 & $\bm{94.5} \pm 0.5$ & $92.7 \pm 0.5$ & $\bm{94.5} \pm 0.5$ \\
    \bottomrule
    \end{tabular}}
\end{table}

\begin{table}[ht]
    \centering
    \caption{\emph{Classification accuracy ($\uparrow$)} with $s$ shots, $5$ ways, and $100$ evaluation tasks each containing $10$ query images, averaged over $20$ random trials. Flag types for `Flag' are $(s-1,2(s-1))$ and the subspace dimension is $s-1$. Baselines see features only from $f_\Theta$.}
    \label{tab:fewshot_app2}
    \resizebox{\columnwidth}{!}{
    \footnotesize
    \begin{tabular}{@{\hskip 4pt}l@{\hskip 4pt}l@{\hskip 4pt}c@{\hskip 7pt} c @{\hskip 7pt}c}
        \toprule
        $s$ & Dataset & Flag & Euc. & Subsp. \\
        \midrule
        \multirow{3}{*}{3} & EuroSat & $\bm{77.7} \pm 1.0$ & $75.9 \pm 0.9$ & $76.8 \pm 1.1$ \\
         & CIFAR-10 & $\bm{59.6} \pm 1.0$ & $58.4 \pm 0.8$ & $58.5 \pm 0.9$ \\
         & Flowers102 & $\bm{90.2} \pm 0.7$ & $87.9 \pm 0.9$ & $88.8 \pm 0.8$ \\
        \cline{1-5}
        \multirow{3}{*}{5} & EuroSat & $\bm{81.8} \pm 0.7$ & $79.8 \pm 0.8$ & $80.8 \pm 0.8$ \\
         & CIFAR-10 & $\bm{65.2} \pm 0.9$ & $64.5 \pm 1.0$ & $63.8 \pm 0.9$ \\
         & Flowers102 & $\bm{93.2} \pm 0.5$ & $91.1 \pm 0.6$ & $92.0 \pm 0.5$ \\
        \cline{1-5}
        \multirow{3}{*}{7} & EuroSat & $\bm{83.9} \pm 0.8$ & $81.7 \pm 0.8$ & $82.9 \pm 0.8$ \\
         & CIFAR-10 & $\bm{68.0} \pm 0.7$ & $67.9 \pm 0.8$ & $66.7 \pm 0.7$ \\
         & Flowers102 & $\bm{94.5} \pm 0.5$ & $92.3 \pm 0.5$ & $93.4 \pm 0.5$ \\
        \bottomrule
    \end{tabular}}
    \vspace{-3mm}
\end{table}

\flushcolsend

\end{document}